\documentclass{article}

 \usepackage[main, final]{neurips_2026}

\usepackage[utf8]{inputenc} 
\usepackage[T1]{fontenc}    
\usepackage{hyperref}       
\usepackage{url}            
\usepackage{booktabs}       
\usepackage{amsfonts}       
\usepackage{nicefrac}       
\usepackage{microtype}      
\usepackage{xcolor}         
\usepackage{mathtools}
\usepackage{dsfont}
\usepackage{enumerate}
\usepackage{float}
\usepackage{amsmath}
\usepackage{amsthm}
\usepackage{algorithm}      
\usepackage{algpseudocode}
\usepackage{caption}
\usepackage{varwidth}
\usepackage{tabularx}
\usepackage{enumitem}
\newtheorem{theorem}{Theorem}
\newtheorem{prop}{Proposition}
\newtheorem{defn}{Definition}
\newtheorem{cor}{Corollary}

\theoremstyle{remark}

\usepackage{xcolor}
\usepackage{array}
\usepackage{wrapfig}
\newcolumntype{L}[1]{>{\raggedright\arraybackslash}p{#1}}
\def \R {\mathbb{R}}
\def \Ghat {\hat{G}}

\def \L {\mathcal{L}}
\def \PSD {\textnormal{PSD}}

\title{It\^o maps for any-step SDEs}
\algnewcommand\Input{\item[\textbf{Input:}]}
\algnewcommand\Output{\item[\textbf{Output:}]}

\author{
  Zhengkai Pan \\ 
  Harvard Uni.\\
  \texttt{zpan@math.harvard.edu} \\
  \And
  Peter Potaptchik \\
  Harvard Uni., Uni. of Oxford \\
  \texttt{potaptchik@stats.ox.ac.uk} \\
  \And
  Wenxi Yao \\
  Harvard Uni. \\
  \texttt{wxyao@math.harvard.edu} \\
  \And
  Michael S. Albergo $\dagger$\\
  Harvard Uni., Kempner Institute \\
\texttt{malbergo@fas.harvard.edu} \\
  \And
  Jakiw Pidstrigach $\dagger$ \\
  Uni. of Oxford, Gridmatic \\
  \texttt{jakiw@gridmatic.com}}

\begin{document}

\maketitle\vspace{-1.5em}
\begin{abstract}\vspace{-0.5em}
Recent one-step generative models accelerate sampling by learning deterministic flow maps of the underlying dynamics. These methods rely on learning from ordinary differential equations, leaving open how to define an exact distillation procedure for stochastic dynamics. We introduce the \textbf{It\^o map}, an any-step stochastic flow map that takes an intermediate state and Brownian path and predicts future states in a single pass. The It\^o map formulation yields novel estimators for inference-time control by providing cheap, differentiable access to posterior samples. Empirically, It\^o maps produce diverse, conditionally valid endpoint samples from fixed intermediate states and support strong steering performance on synthetic and image-generation benchmarks. These results establish any-step SDE integration as a useful primitive for posterior sampling and stochastic control.
\end{abstract}
\vspace{-2em}\begin{figure}[H]
    \vspace{0.5em}
    \centering
    \includegraphics[
        width=0.94\linewidth,
        keepaspectratio,
        trim={0 0.4cm 0 0.2cm},
        clip
    ]{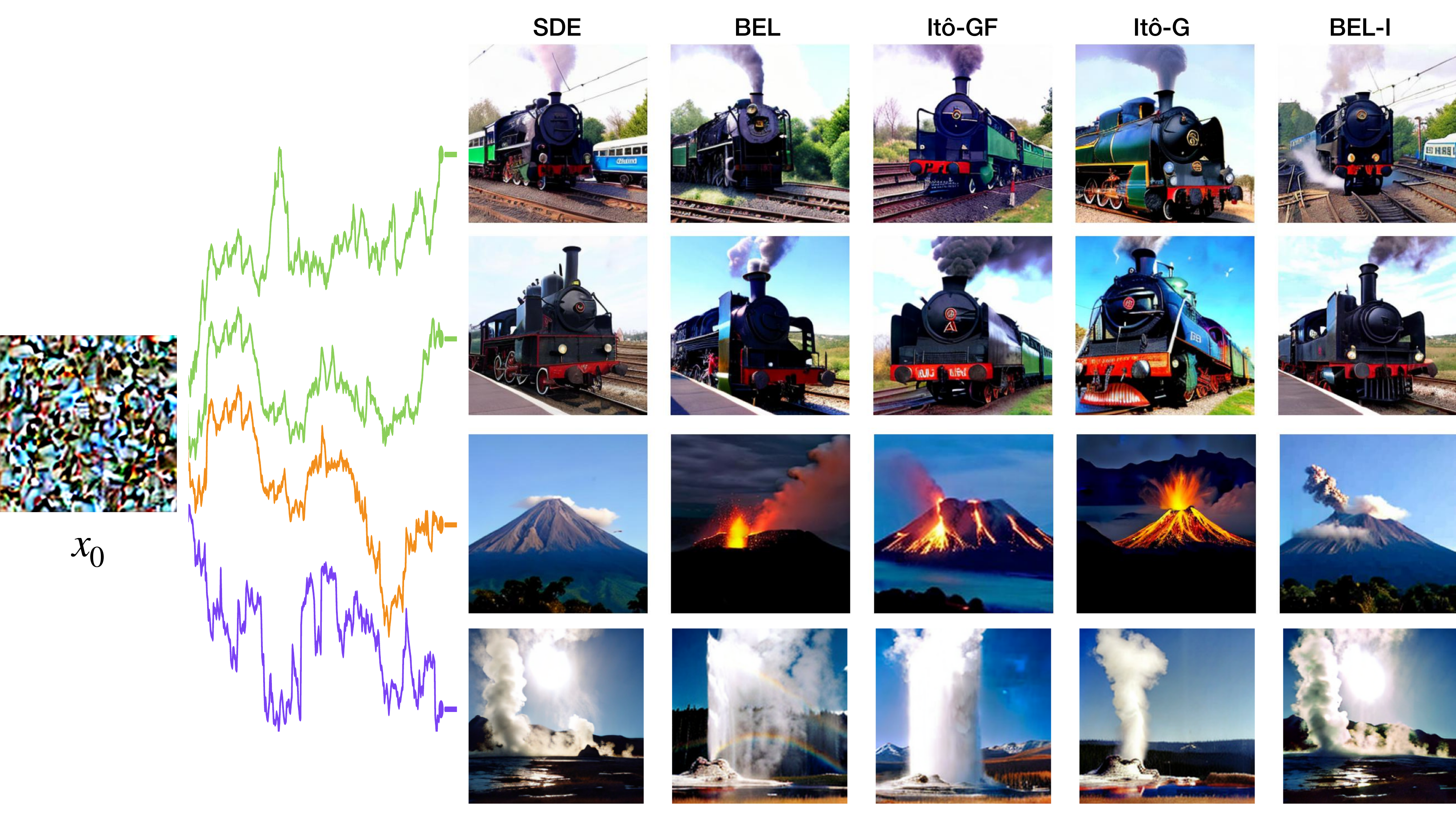}
    \caption{Inference-time steering with ImageReward. Prompt for \textbf{top 2 rows:} \textit{a steam locomotive with steam billowing}; for the \textbf{third row:} \textit{a volcano erupting with magma}; for the \textbf{fourth row:} \textit{a geyser erupting with a towering steam plume}. Each row uses a single shared Brownian trajectory.}
    \label{fig:frontpage}
\end{figure}
\vspace{-2em}
\section{Introduction}\vspace{-1em}
Many state-of-the-art generative models transport a simple reference to a complex target by many small steps. The results are impressive, but inference is costly: it requires many sequential passes through large neural networks. Recent works have successfully turned many-step ODE samplers into one-step samplers by learning their flow \citep{consistency,howtobuild,dmf,tvm}. However, most of that story is fundamentally deterministic. These methods compress continuous-time dynamics into a direct map, typically aligned with the probability-flow ODE (PF-ODE) viewpoint \citep{song2020score,ddpm}. In contrast, score-based generative modeling is naturally defined through a denoising SDE \citep{song2020score,ddpm}, whose sampling dynamics remains stochastic, even though it shares the same marginals with the PF-ODE at each time. As a result, current one-step models largely learn fast transports, not fast stochastic transition operators. More importantly, deterministic one-step maps cannot represent the endpoint posterior from an intermediate noisy state.

The key missing object is the one-shot conditional law \begin{equation}
    p_{1\mid t}(\cdot\mid x)
\end{equation} from an intermediate state $x_t$ to a clean endpoint $x_1$. For many downstream applications, this conditional law is often the quantity of interest since it captures the residual uncertainty, supports posterior sampling and hence provides the expectation terms required by many optimal control estimators in inference-time steering. Deterministic one-step maps learn a point prediction of the endpoint; It\^o maps learn a sampler for the endpoint conditional law $p_{1\mid t}(\cdot\mid x)$.

Unlike ODEs -- which define deterministic flows learnable as a function of the initial condition -- SDE trajectories are not deterministic functions of the initial condition alone. Our key observation is that if we also condition on the Brownian \emph{path}, the SDE trajectory is deterministic and therefore it admits a single-step prediction. While the Brownian path is formally infinite-dimensional, we show that with a careful extraction of the Brownian information, 5 dimensions per data dimension suffice for accurate one-shot prediction in our lower-dimensional experiments, and 10 dimensions for ImageNet-256.

In this paper, we propose to learn the \emph{It\^o map}: a map that takes as input a state $X_s$ and the realization of Brownian motion $(W_t)_{t\in[0,1]}$, and predicts $X_u$ on the \emph{same path} -- in a single forward pass for any $s<u$. We study the general It\^o map between arbitrary times on the same stochastic path, with the one-shot endpoint case $s=t,u=1$ as the main generative setting of interest. Our method preserves the parts of diffusion that deterministic one-step maps discard: conditional uncertainty, posterior diversity, and the stochastic structure needed for control.
Our \textbf{main contributions} include:\begin{itemize}
    \item We introduce \textbf{It\^o maps} for any-step generative SDE integration, which can be thought of as path-wise stochastic flow maps whose stochasticity is driven by the Brownian motion of the underlying SDE. This gives a principled model of stochastic transition laws between arbitrary times, with the endpoint conditional law $p_{1\mid t}(x_1\mid x_t)$ as the main generative case.
    \item We develop a practical learning framework for It\^o maps using low-dimensional representations of Brownian information, making single-pass stochastic prediction practical in high-dimensional generative modeling.
    \item We derive \textbf{novel optimal control estimators} from the It\^o map formulation, including It\^o-G, It\^o-GF, and the BEL-family estimators which rely on Brownian path structure. These estimators are not naturally available in deterministic one-step models or stochastic one-step models that inject noise only at the state level.
    \item We demonstrate the benefits of this formulation on posterior sampling and inference-time steering. On matched benchmarks, It\^o-G is a strong estimator when reward gradients are available, outperforming many baseline estimators, including MFM-G; the gradient-free estimators are promising alternatives with mixed strength depending on the setting.
\end{itemize}
\vspace{-1em}
\begin{figure}[H]
  \centering \includegraphics[scale=0.15]{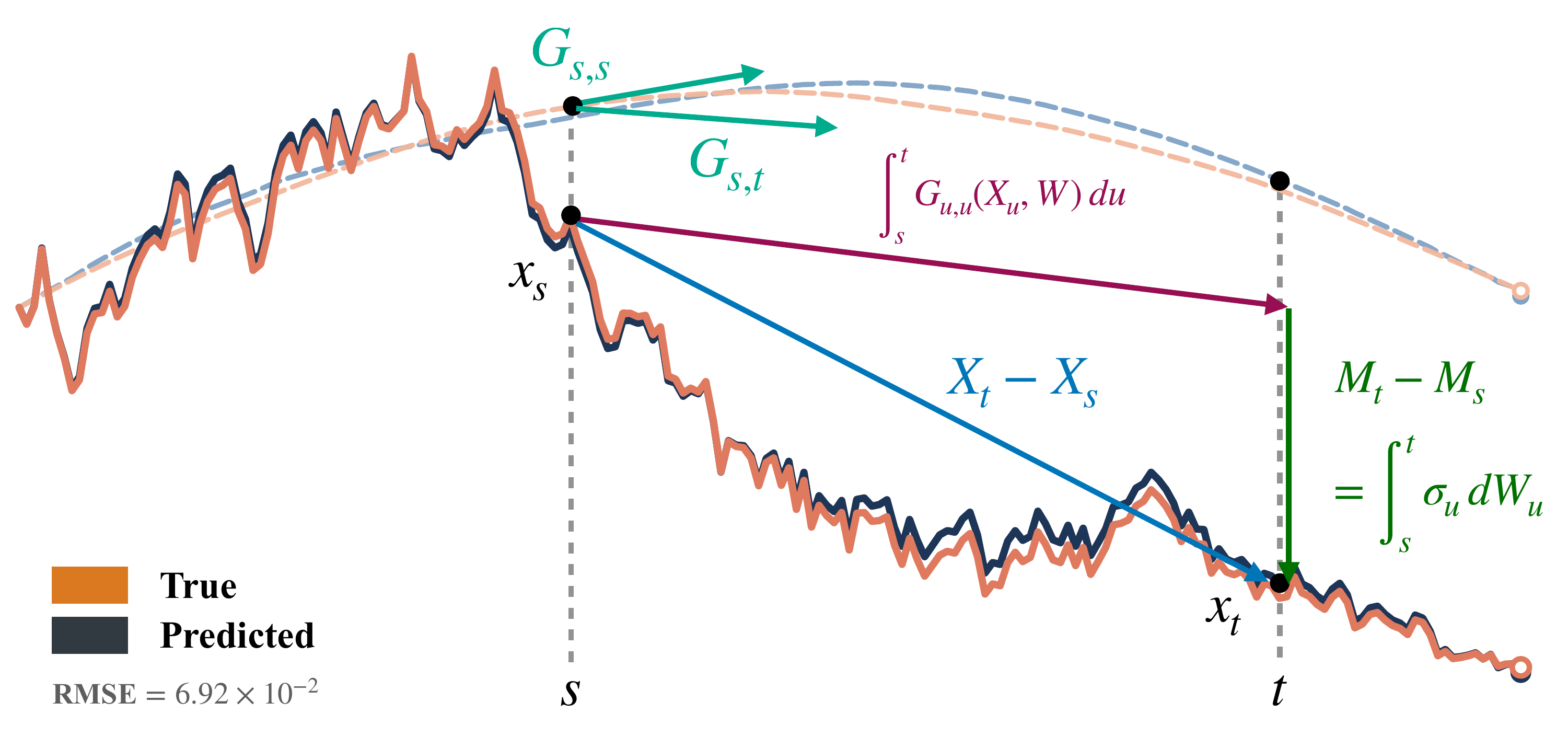}
  \caption{It\^o map trajectory $(\hat{X}_{0,t}(X_0,W))_{t\in[0,1]}$ vs. true SDE roll-out for a 1D Gaussian mixture. We itemize the various terms in (\ref{avgvelo}) to distinguish between the uniquely deterministic and stochastic parts of the expression. } \label{1dgmm}
\end{figure}
\subsection{Related Works}
\paragraph{Dynamical transport and distillation.} Recent generative modeling methods describe the evolution from noise to data either through denoising dynamics, as in diffusion and score-based models \citep{ddpm,song2020score}, or through simulation-free transport formulations, as in flow matching, rectified flow, and stochastic interpolants \citep{lipman2023fm,rectifiedflow,albergo23si}. More recently, consistency models and flow-map distillation methods have shown how to compress such dynamics into few-step or one-step samplers \citep{consistency,howtobuild,dmf,tvm}. However, these approaches learn  velocity fields associated with the evolving law, or deterministic fast samplers. In contrast, our goal is to learn a same-path stochastic flow map conditioned on Brownian information, so that the model represents the endpoint conditional law rather than only a point prediction or average transport.

\paragraph{Inference-time steering and stochastic control.}
A related line of work studies inference-time steering toward conditioned or reward-tilted targets. Existing methods broadly fall into two categories: methods that approximate the tilted dynamics using surrogate posteriors or guidance terms, including DPS \citep{dps}, FreeDoM \citep{freedom}, MPGD \citep{mpgd}, Universal Guidance \citep{ugd}, and LGD \citep{lgd}; and particle-based or search-based methods, for which we refer to Uehara et al. \citep{uehara2025ita} for a recent overview. A closely related direction formulates steering through estimation of a value function or its gradient, as in stochastic-control and neural-sampling approaches \citep{vargas2022nsf,akhoundsadegh2024idem}; recent one-step methods such as Meta Flow Maps \citep{mfm} likewise use value-function estimation for reward alignment, but do so on top of fast samplers with exogenous stochasticity. Our work is related in downstream objective, but differs in the object being learned: we learn a same-path stochastic flow map conditioned on Brownian information, and then build control estimators on top of it.
\section{It\^o Maps via Stochastic Interpolants}

In what follows, we first recall the theory of deterministic and stochastic flows to derive theoretically sound objectives for learning the It\^o map. Then we propose a low-dimensional representation of the Brownian path to be injected into the neural network. This leads to an algorithm for training It\^o maps.
\subsection{Background: Stochastic Interpolants and Flow Maps}
\paragraph{Stochastic interpolants.} The common goal in generative modeling is to learn a transport from a noise distribution $ p_0$ to a target data distribution $x_1\sim p_1$. One way to describe this transport is via a time-dependent velocity field $b:\mathbb{R}_t\times\mathbb{R}_x^n\to\mathbb{R}^n$ whose trajectories satisfy the continuity equation:
\begin{align}
\label{eq:ode}
    \dot x_t = b_t(x_t), \quad x_0 \sim p_0, \quad \textrm{and} \quad \partial_t p_t + \nabla \cdot (b_ tp_t) = 0 
\end{align} 
Following stochastic interpolants and flow-matching viewpoints, it is helpful to introduce an auxiliary random process that has the same marginal law $p_t$ as (\ref{eq:ode}), but can be sampled without integrating trajectories. Let $(x_0,x_1)$ be drawn from a joint distribution $\rho(x_0,x_1)$ whose marginals are $p_0$ and $p_1$, and define the linear interpolant \citep{abve_jmlr,lipman2023fm}
\begin{align}
\label{eq:si}
    I_t \coloneqq (1-t)x_0 + t x_1, \quad (x_0, x_1) \sim \rho(x_0, x_1).
\end{align}
By construction, $\text{Law}(I_t) = p_t$ and the velocity field which solves (\ref{eq:ode}) can be written as
\begin{align}
\label{eq:bt}
b_t(x) = \mathbb E[\dot I_t | I_t = x] = \mathbb{E}[x_1-x_0\mid I_t=x].\end{align}
This yields the standard regression objective for learning $b_t$: $\mathcal L_{\text{SI}}[ \hat b] = \int_0^1 \mathbb E[ \| \hat b_t(I_t) - \dot I_t \|^2] dt$.
The discussion above gives a deterministic transport viewpoint. Our goal, however, is not only to model this deterministic transport, but also to retain the stochastic structure of the generative process.

\paragraph{Flow maps.} In the deterministic setting, the evolution from time $s$ to time $ t$ is fully determined by the current state $X_s$, and can therefore be summarized by a flow map. Concretely, one may define a deterministic average velocity $v_{s,t}(X_s)$ through \begin{equation}
    X_t=F_{s,t}(X_s)=X_s+(t-s)v_{s,t}(X_s)
\end{equation}
Recent one-step and few-step samplers \citep{consistency,ctm,shortcut,howtobuild,meanflow,dmf} can be viewed as learning such deterministic
flow maps, or equivalently finite-time average velocities.

\subsection{Any-Step SDEs via Stochastic Flow Maps}
\paragraph{From generative ODE to SDE.} We can add stochasticity to the deterministic transport in (\ref{eq:ode}) without changing its time marginals $p_t$. We augment the ODE dynamics with a Langevin term and consider the following SDE~\citep{song2020score,abve_jmlr}:
\begin{equation}
    dX_t = b_t(X_t)\,dt + \frac{\sigma_t^2}{2}\nabla \log p_t(X_t)\,dt + \sigma_t\,dW_t.
    \label{eq:gensde}
\end{equation}

\paragraph{Stochastic flow maps.} For the SDE in (\ref{eq:gensde}) and for $s\leq t$, a future state $X_t$ is no longer determined by the current state $X_s$ alone. It also depends on the realized Brownian path over the interval $[s,t]$. This suggests the stochastic analogue of a flow map should take both the current state and the driving Brownian motion as input, and predict the future state on the same stochastic trajectory. We call this object the \textbf{It\^o map} $\Phi_{s,t}$, which maps $(X_s,W)$ to $X_t$ along the same Brownian path $W$. 

Note that $X_t$ can be written as the sum of a finite variation process $G_{s,t}$ and a martingale $M_t\coloneqq \int_0^t\sigma_u\,dW_u$:
\begin{equation}
X_t - X_s
=
\underbrace{
\int_s^t
\left(
b_u(X_u)
+
\frac{\sigma_u^2}{2}\nabla \log p_u(X_u)
\right)\,du
}_{=: ~(t-s)G_{s,t}}
+
\underbrace{
\int_s^t \sigma_u\,dW_u
}_{= ~M_t - M_s}.
\end{equation}
If we want to transport from $X_s$ directly to $X_t$, we need to know both the \textit{average velocity} $G$:
\begin{equation}
    G_{s, t}(X_s, W) \coloneqq \frac{X_t - X_s - M_t + M_s }{t-s}  = \frac{1}{t-s} \int_s^t b_u(X_u)+\frac{\sigma^2_u}{2}\nabla \log p_{u}(X_u) du \label{avgvelo},
\end{equation}
as well as the martingale part $M_t-M_s = \int_s^t \sigma_u dW_u$. However, the latter is a measurable function of the Brownian path $W$, so in practice we only need to learn $G$ and can reconstruct $X$ via\footnote{The equality holds for almost every path $W$, for all $s$ and $t$. The existence of such a map $F$ is a result of the theory of stochastic flows \citep{kunita1990stochastic}.}:
 \begin{equation}\label{eq:ansatz}
   X_t= F_{s,t}(x, W) = x +(t-s) G_{s,t}(x, W) + (M_t-M_s).
\end{equation}
\subsection{Learning Stochastic Flow Maps via Self-Distillation}
In practice, we split the task of learning $G$ into two parts: first learning the diagonal $G_{t,t}$, and then learning the off-diagonal $G_{s, t}$ for $s\neq t$.

\paragraph{Learning the diagonal. } Taking $s\uparrow t$ in equation (\ref{avgvelo}), we obtain
\begin{equation}
\label{eq:diag}
    G_{t, t}(X_t, W) = b_t(X_t) + \frac{\sigma_t^2}{2}\nabla \log p_{t}(X_t).
\end{equation}
Here $b_t$ can be learned via standard approaches \citep{lipman2023fm,rectifiedflow,albergo23si} and $\nabla \log p_t$ can be reconstructed from that \citep{abve_jmlr}. Alternatively, one can also learn $G_{t,t}$ directly, since it can be written as a conditional expectation:
\begin{equation}
    G_{t,t}(x,W) = \mathbb{E}_{(X_0,X_1)}\left[ X_1-\left(1+\frac{\sigma_t^2}{2(1-t)}\right)X_0 \mid I_t=x\right].\label{eq:sdeinterpdrift}
\end{equation}
 We leave the proof to Appendix \ref{proofdiag}. Note that $G_{t,t}$ is clearly independent of $W$, so we may as well denote it by $G_{t,t}(x)$.
For training stability, we remove the singularity at $t=1$ by setting $\sigma_t=\sqrt{2(1-t)}$. This in particular yields \begin{equation}
    G_{t,t}(x,W) = \mathbb{E}[X_1-2X_0\mid I_t=x].
\end{equation} Therefore we can learn $G_{t, t}(\cdot, \cdot)$ via 
\begin{equation}
    \mathcal{L}_{\text{SI}}(\hat{G}) = \int_0^1\mathbb{E} \|\hat{G}_{u,u}(I_u) - (X_1-2X_0)\|^2\,du.
\end{equation}
To learn a (stochastic) flow map, one also needs off-diagonal training (i.e. training $G_{s,t}$ for $s\neq t$) to enforce consistency across time intervals in addition to the diagonal training of $G_{t,t}$.
\paragraph{Learning the off-diagonal via Lagrangian self-distillation (LSD).} The Lagrangian perspective \citep{lagrangian, howtobuild} asks how the process $F_{s,t}(X_s, W)$ evolves with time $t$, and matches this evolution to the desired dynamics $X_t = \Phi_{s,t}(X_s, W)$. Recall that the average velocity underlying the It\^o map $G_{s,t}(x,W)$ can be written as (\ref{avgvelo}). Differentiating (\ref{avgvelo}) with respect to $t$ yields:
\begin{equation}
    b_t(F_{s,t}(x,W))+\frac{\sigma_t^2}{2}\nabla \log p_t(F_{s,t}(x,W))=G_{s,t}(x,W)+(t-s)\partial_tG_{s,t}(x,W)
\end{equation}
which must hold for all $s\leq t$. For $s\uparrow t$, it recovers the diagonal condition. This identity motivates an off-diagonal training objective: we enforce that the diagonal drift predicted at time $t$ matches the RHS computed from $G_{s,t}$ and its $t$-derivative. Concretely, we define the LSD consistency loss
\begin{equation}
\label{eq:lsd_pinn}
    \mathcal{L}_{\text{LSD}}(F) = \int_{s \leq t} \mathbb E  [\|G_{t,t}(F_{s,t}(X_s,W), W) - G_{s,t}(X_s, W) - (t-s) \partial_t G_{s,t}(X_s, W) \|^2]dsdt.
\end{equation}
In addition, the semigroup law of flows gives rise to the progressive self-distillation loss:
\begin{equation}
    \mathcal{L}_{\text{PSD}}(F) = \int_{s \leq u \leq t} \mathbb E \left[\left \|F_{s, t}( X_s, W) - F_{u, t}(F_{s, u}(X_s, W), W)\right \|^2\right ]dsdudt \label{lsd_obj}
\end{equation}
We leave the detailed explanations and the proof that LSD and PSD are both valid off-diagonal objectives for It\^o map learning to the appendix. Readers may also be familiar with Eulerian-type objectives, including mean flow; we prove in Appendix \ref{esd} that naive generalization of the deterministic Eulerian objectives \textit{cannot} be used to train It\^o maps.
\subsection{Extracting Brownian Information with Karhunen–Loève Expansion.}\label{KL}
Brownian sample paths can be viewed as elements of the infinite-dimensional path space $L^2([0,T],\mathbb{R}^d)$. A direct discretization of this path space with $N$ time points would represent a $d$-dimensional Brownian path using $dN$ inputs, which creates an unnecessarily large conditioning space and can hinder learning of the state dependence. Therefore, we need to use a lower dimensional representation of the Brownian information. A result by \citep{kosambi,karhunen, loeve} states that for a standard Brownian trajectory $W_t$ on $[0,T]$, $W_t$ admits an expansion \begin{equation}
        W_t= \sum_{n=1}^{\infty} \frac{T}{(n-\frac{1}{2})\pi}\,\xi_n\,\sqrt{\frac{2}{T}}\sin\left(\frac{(n-\frac{1}{2})\pi t}{T}\right)
    \end{equation}
    where the KL coefficients $\xi_n$ are i.i.d. standard Gaussian random variables. Note that as $n$ increases, the amplitude of the $n$-th term decreases. Thus, the first few KL coefficients already capture most of the Brownian information and can be used to form a low-dimensional representation of the Brownian path. See Appendix \ref{klthm} for proof of the KL theorem. In our MNIST and ImageNet runs, we observe that using local KL coefficients on $[s,1]$ leads to better performance. In this paper, we choose \textbf{$K=5$} KL modes by default. For ImageNet-256 only, we use 10 KL coefficients.
\begin{center}
\begin{minipage}{0.98\linewidth} 
  \begin{algorithm}[H]
\caption{Training from scratch with Lagrangian self-distillation.}
\label{alg:LSD}
\begin{algorithmic}[1]
  \Input $p_0 = p_{\text{noise}},\; p_1 = p_{\text{data}}$; model $\hat{G}_{s,t}(x,W)$
  \Repeat
    \State \textbf{Compute diagonal loss:}
    \State Sample $t \sim U[0,1]$; simulate $X_0 \sim p_0$ and $X_1 \sim p_{\text{data}}$
    \State $I_t \gets tX_1 + (1-t)X_0$
    \State $\mathcal{L}_{\mathrm{SI}} \gets \|\hat{G}_{t,t}(I_t,\mathrm{zeros}) - (X_1 - 2X_0)\|_2^2$

    \State \textbf{Compute consistency loss:}
    \State Sample $(s,t)\sim U[0,1]^2$ and reorder so that $s<t$; sample $X_0\sim p_0$, $X_1\sim p_{\text{data}}$
    \State Simulate Brownian trajectory $(W_l)_{l\in[0,1]}$ or $(M_l)_{l\in[s,1]}$ and extract Brownian features $W$
    \State $I_s \gets sX_1 + (1-s)X_0$
    \State $X_t \gets I_s + (t-s)\hat{G}_{s,t}(I_s,W) + M_t - M_s$ for $M_t - M_s = \int_s^t \sigma_u\,dW_u$
    \State $\mathcal{L}_{\mathrm{LSD}} \gets
      \bigl\|\hat{G}_{s,t}(I_s,W)
      +(t-s) \partial_t \hat{G}_{s,t}(I_s,W)
      - \operatorname{stopgrad}\bigl(\hat{G}_{t,t}(X_t,W)\bigr)\bigr\|_2^2$
    \State $\mathrm{loss} \gets \mathcal{L}_{\mathrm{SI}} + \lambda \mathcal{L}_{\mathrm{LSD}}$
    \State Update model by taking one optimizer step
  \Until{convergence}
  \Output Trained It\^o map $ \hat{X}_{s,t}(x,W) = x + (t-s)\hat{G}_{s,t}(x,W) + M_t - M_s$
    \end{algorithmic}
\end{algorithm}
\end{minipage}
\end{center}

\section{Reward Alignment}
\subsection{Doob's $h$-Transform and Steering Algorithm}
A natural downstream application of It\^o maps is inference-time steering. Given a reward function $r:\mathbb{R}^d\to\mathbb{R}$ defined on the ambient space containing the data manifold, our goal is to steer the terminal samples toward high-reward regions while preserving the stochastic structure of the sampler. The \emph{reward-tilted distribution} is defined as
\begin{equation}
    p_{\text{reward}}(x)\propto p_{\text{model}}(x) e^{r(x)} \label{eq:reward}
\end{equation}
where $p_{\text{model}}$ is the terminal distribution of our It\^o map generation, approximating $p_{\text{data}}$. In stochastic optimal control, this problem is governed by the \emph{value function} $V_t(x)\coloneqq\log\mathbb{E}[\exp{r(X_1)\mid X_t=x}]$. By Doob’s $h$-transform \citep{doob1957,daipra}, the required drift correction to (\ref{eq:gensde}) is precisely the \textit{optimal control} term $\nabla V_t(x)$, yielding the controlled SDE:
\begin{equation}
    dX_t^*= \left(b_t(X_t^*)+\frac{\sigma_t^2}{2}\nabla \log p_t(X_t^*)+\sigma_t^2 \nabla V_t(X_t^*)\right)dt + \sigma_tdW_t. \label{controlsde}
\end{equation}
Under the new dynamics of $(X_t^*)_{t\in [0,1]}$, the marginal laws are given by \begin{equation}
    p_t^*(x)\propto p_t(x)e^{V_t(x)} \label{tiltmarg}
\end{equation}
when initialized from $p_0^*\propto p_0e^{V_0}$. If we cut the time interval $[0,1]$ into $N$ pieces, $\{0,\frac{1}{N},...,\frac{N-1}{N},1\}$, define $t_i=\frac{i}{N}$ and denote by $\hat{v}_{t}(x)$ the estimator for $\nabla V_t(x)$, steering is performed by rolling out the controlled dynamics using the learned It\^o map along a fixed Brownian trajectory:
\vspace{-1em}

{\centering
\begin{minipage}{0.8\linewidth}
\begin{algorithm}[H]
\caption{Inference-time steering roll-out.}
\label{alg:steering-roll-out}
\begin{algorithmic}[1]
  \Input Fix one $x_0 \sim p_0^*$ and one Brownian trajectory $(W_t)_{t\in[0,1]}$;
  \Statex \hspace{1.8em}trained It\^o map $\hat{G}_{s,t}(x,W)$ and optimal control estimator $\hat{v}_t(x)$
  \For{$i = 0,1,\dots,N-1$}
    \State $x_{t_{i+1}} \gets x_{t_i}
    + \Delta t \Bigl[\hat{G}_{t_i,t_i}(x_{t_i},W)
    + \sigma_{t_i}^2 \hat{v}_{t_i}(x_{t_i})\Bigr]
    + M_{t_{i+1}} - M_{t_i}$
  \EndFor
  \Output $x_1$
\end{algorithmic}
\end{algorithm}
\end{minipage}
\par}
In practice, we sample $x_0\sim p_0$ and may need to rescale the reward function for better steering performance. For $\lambda\in\mathbb{R}_{>0}$, the rescaled tilted distribution is given by $p_1^*(x)\propto p_{\text{model}}e^{\lambda\cdot r(x)}$.\vspace{-0.5em}
\subsection{Estimators of $\nabla V_t(x)$}\vspace{-0.5em}
The main practical problem in this section is therefore to estimate $\nabla V_t(x)$. Because It\^o maps provide conditional endpoint samples from fixed intermediate states and retain Brownian path information, they support both gradient-based and gradient-free estimators of this control. When the gradient of the reward function is available, the identity \begin{equation}
    \nabla_xV_t(x)=\nabla_x\log\mathbb{E}[e^{r(X_1)}\mid X_t=x] \end{equation} leads directly to the following Monte Carlo estimator.
\begin{defn}
At time $t_i$ and state $x_i$, the It\^o-G estimator (``$G$'' stands for gradient) for $\nabla V_t(x)$ is
\begin{equation}
    \nabla\log \frac{1}{Z}\sum_{j=1}^Z\exp(r(\hat{X}(t_i,1,x_i,W^j)))
\end{equation}
where $W^j$ are the Brownian trajectories that are newly simulated on $[t_i,1]$.
\end{defn}
It\^o-G is simple and uses only conditional endpoint samples, but requires access to reward gradients. When the gradient of the terminal weight ($\exp\circ\,r$) is not available or is poorly behaved (e.g. spiky, non-differentiable, or even singular), we propose the following gradient-free estimators.
\begin{defn} At time $t_i$ and state $x_i$, the It\^o-GF estimator for $\nabla V_{t}(x)$ is
\begin{equation}
    \frac{2}{\sigma_{t_i}^2}\frac{1}{1-t_i}\cdot\left( \frac{\sum_{j=1}^Z\hat{X}(t_i,1,x_i,W^j)\exp(r(\hat{X}(t_i,1,x_i,W^j)))}{\sum_{j=1}^Z\exp (r(\hat{X}(t_i,1,x_i,W^j)))}-\frac{1}{Z}\sum_{j=1}^Z\hat{X}(t_i,1,x_i,W^j)\right)
\end{equation}\vspace{-0.5em}
where $W^j$ are the Brownian trajectories that are newly simulated on $[t_i,1]$.
\end{defn}
It\^o-GF avoids differentiating through the reward by comparing reward-weighted and unweighted endpoint expectations. See \ref{estprf} for the proof. A second gradient-free route is to exploit the Brownian path structure directly, leading to a representation of the optimal control in terms of pathwise Jacobians and Brownian increments. BEL-family estimators can also deal with singular terminal weights.
\begin{theorem} 
    [Generalized BEL formula, \citep{BEL}] Let $X_t$ be the uncontrolled dynamics with Brownian motion W (with path measure $\mathbb{P}$). Let $s\in[0,T)$ and $J_{t\mid s}\coloneqq \nabla_{X_s}X_t$ be the pathwise Jacobian for any $t\in (s,T)$. Let $F=\exp\circ\, r$ be the terminal weight and $\mathbb{Q}$ be the terminally tilted path measure. Then for any weighting $\alpha_{\cdot\mid s}$ defined on $[s,T]$ with $
        \int_s^T\alpha_{u\mid s}\,du=1
    $, the optimal control satisfies:
    \vspace{-0.5em}
    \begin{equation}
        \mathcal{V}_s(x)=\nabla_x\log\mathbb{E}_{\mathbb{P}}[F(X_T)\mid X_s=x]=\mathbb{E}_{\mathbb{Q}}\left[\int_s^T\alpha_{t\mid s}J^{\top}_{t\mid s}\sigma_t^{-1}dW_t \mid X_s=x\right]
    \end{equation}\vspace{-0.5em}
    where $\sigma_t$ is our diffusion coefficient. \label{bel}
\end{theorem}
See \ref{belprf} for the proof to Theorem \ref{bel}. In our setup, the terminal time $T=1$, hence yielding a family of estimators indexed by the weighting function $\alpha_{t\mid s}$:
    \begin{defn}
    At time $t_i$ and state $x_i$, the BEL family estimators for $\nabla V_t(x)$ are:
    \begin{equation}
        \frac{ \sum_{j=1}^Z\exp{(r(\hat{X}(t_i,1,x_{i},W^j)))}\cdot\sum_{k=i}^{N-1} J_{t_k\mid t_i}^{j,\top}\Delta W^j_{t_k:t_{k+1}}\frac{\alpha_{t_k\mid t_i}}{\sigma_{t_k}}}{\sum_{j=1}^Z\exp{(r(\hat{X}(t_i,1,x_{i},W^j)))}}
    \end{equation}
\end{defn}
 where $(\alpha_{t\mid t_i})_{ti\leq t\leq 1}$ is any weight with $\int_{t_i}^1\alpha_{u\mid s}\,du=1$. We choose $\alpha_{t\mid s}\propto\sigma_t$ as the default BEL estimator. A particularly cheap choice is to have $\alpha_t$ only supported on the first time interval after $t_i$:
\begin{defn}
    [BEL-First/BEL-I] If we set $\alpha_{t_i} = \frac{1}{\Delta t} \mathds{1}_{[t_i,t_{i+1}]}$, and applying a left-point discretization of the BEL integral yields the BEL-I estimator:
    \begin{equation}
        \frac{ \sum_{j=1}^Z\exp{(r(\hat{X}(t_i,1,x_{i},W^j)))}\cdot \Delta W_{t_i:t_{i+1}}\frac{\mathds{1}_{[t_i,t_{i+1}]}}{\Delta t \cdot \sqrt{2(1-t_i)}}}{\sum_{j=1}^Z\exp{(r(\hat{X}(t_i,1,x_{i},W^j)))}}.
    \end{equation}
\end{defn}
In summary, It\^o-G is the simplest option when reward gradients are available; It\^o-GF removes reward gradient dependence using conditional endpoint averages; and the BEL family further exploits Brownian path structure to obtain pathwise gradient-free estimators.
\vspace{-0.5em}
\section{Experiments}\vspace{-0.5em}
We evaluate It\^o maps on two tasks: same-path stochastic prediction and inference-time steering. Same-path prediction tests whether the learned map reproduces SDE transitions under matched Brownian trajectories, while steering tests whether its conditional endpoint samples can estimate control directions for reward-tilted generation. We start with low-dimensional Gaussian mixtures and then move to MNIST and ImageNet.
\vspace{-0.5em}
\subsection{1-Dimensional Gaussian Mixture Model}
We first test the same-path prediction accuracy with a 1D GMM. Starting from a fixed initial state $X_0$ and a fixed Brownian path $W_{[0,1]}$, the drift term is available in closed form. We compare the true SDE roll-out $(X_t)_{t\in [0,1]}$ with the trajectory obtained by one-shot It\^o map predictions from time 0 to any time $t$: $(\hat{X}(0,t,X_0,W))_{t\in[0,1]}$. As shown in Figure \ref{1dgmm}, the predicted trajectory closely follows the true SDE trajectory. 
This experiment serves as a sanity check that the model correctly interprets KL coefficients as Brownian information and reproduces same-path SDE trajectories.
\subsection{2D Gaussian Mixture Model for Reward Alignment}\vspace{-1em}
We use a 2D inverse problem with known posterior to evaluate steering accuracy. Take a Gaussian mixture prior $X$ and make a noisy linear observation $Y$, i.e. $p(y\mid x)=\mathcal{N}(y; Ax,\sigma^2I)$ for a fixed linear map $A:\mathbb{R}^2\to\mathbb{R}$. We steer towards the posterior distribution $p(x\mid y)$, which can be analytically computed for given $A,\sigma$ and $p_x$. This setting provides an exact target distribution, allowing us to compare steering quality quantitatively rather than only visually. In our experiment, we set $X\sim \frac{1}{3}\mathcal{N}((-3,-3), 0.5I)+ \frac{1}{3}\mathcal{N}((0,0), 0.5I)+\frac{1}{3}\mathcal{N}((3,3), 0.5I)$ with $A=[1.2,-0.8]$ and $\sigma=0.2$ and assume that the observation $y=-1.0$. This setup exactly aligns with Meta Flow Maps so that we can directly compare with the reported MFM baselines. We produce 4096 posterior samples, and report the distance of the steered samples from the true posterior distribution using Sliced-Wasserstein distance S-W2 and Maximum Mean Discrepancy (MMD) as metrics.\vspace{-1em}
\begin{figure}[H]
    \centering
    \begin{minipage}[c]{0.6\linewidth}
        \includegraphics[width=\linewidth]{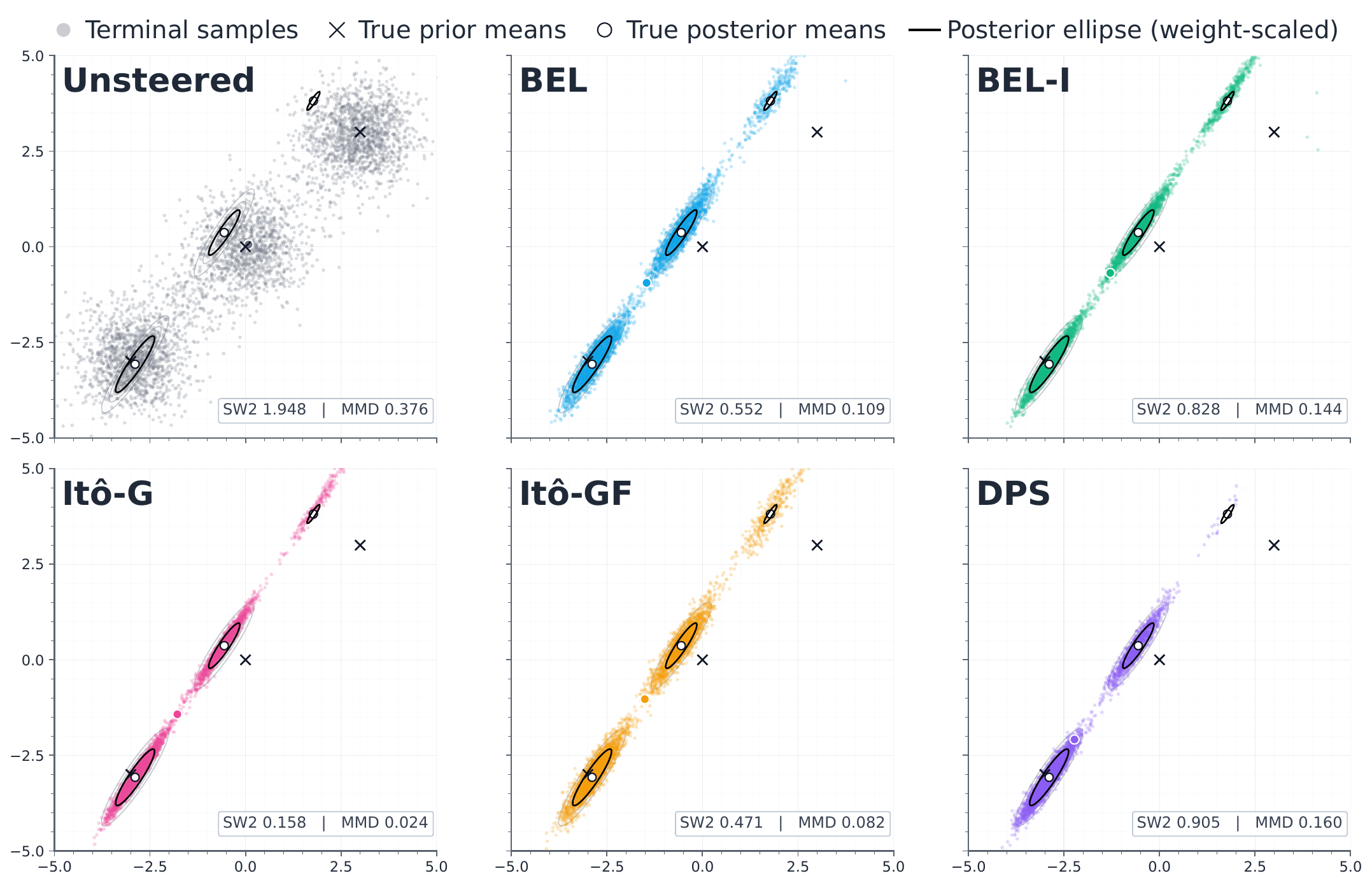}
    \end{minipage}\hfill
    \begin{minipage}[c]{0.4\linewidth}
        \centering
        \renewcommand{\arraystretch}{1}
        \begin{tabular}{lll}
            \toprule
            Estimator & S-W2 & MMD \\
            \midrule
            Unsteered & 1.95  & 0.38  \\
            It\^o-G & \textbf{0.16}  &  0.024 \\
            It\^o-GF &  0.47 &  0.082 \\
            BEL &  0.55 &  0.11 \\
            BEL-I &  0.83 &  0.14 \\
            DPS & 0.91  &  0.16 \\
            MFM-G & 0.51  & 0.027 \\
            MFM-GF & 0.29 &0.0064 \\
            MFM-FT & 0.20 & 0.017 \\
            \bottomrule
        \end{tabular} 
    \end{minipage}
    \caption{2-D GMM steering towards the posterior distribution $p(x \mid y)$. \textbf{Left}: comparison of the terminal states of the 4096 steered particles with different estimators using 128 Monte Carlo samples. \textbf{Right}: a comparison of metrics against MFM and DPS, using 128 MC samples. Reward scale = 1. Here, DPS denotes an MC-based variant; see Appendix \ref{appendb} for details.}
    \label{2dgmm}
\end{figure}
\vspace{-2em}\noindent
\begin{minipage}[t]{0.56\textwidth}
    \vspace{0pt}
    \centering
    \includegraphics[width=\linewidth]{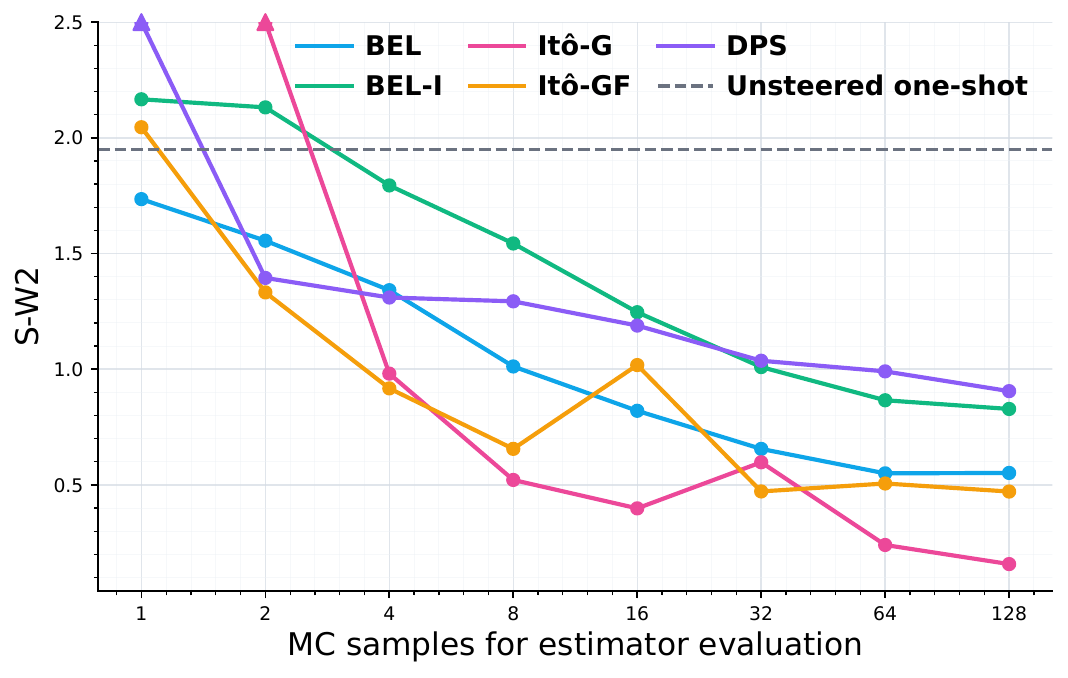}
    \captionof{figure}{S-W2 from the true posterior against the number of MC samples.}
    \label{fig:sw2}
\end{minipage}
\hfill
\begin{minipage}[t]{0.40\textwidth}
    \vspace{0pt}
    Figures~\ref{2dgmm} and~\ref{fig:sw2} show that the It\^o map
    control estimators substantially improve tilted posterior sampling
    over the unsteered roll-out. When reward gradients are available,
    It\^o-G achieves the best S-W2 among the compared methods, with an
    S-W2 of $0.16$ and an MMD of $0.024$ using 128 MC samples. This
    improves over MFM-G on both metrics, while MFM-GF obtains the lowest
    MMD. The gradient-free It\^o-GF, BEL, and BEL-I estimators also
    consistently outperform the DPS-guided roll-out~\citep{dps}.
\end{minipage}
\par\vspace{-1em}
\subsection{MNIST Same-Path Prediction}\vspace{-1em}
\noindent
\begin{minipage}[t]{0.56\textwidth}
    \vspace{0pt}
    \centering
    \includegraphics[width=\linewidth]{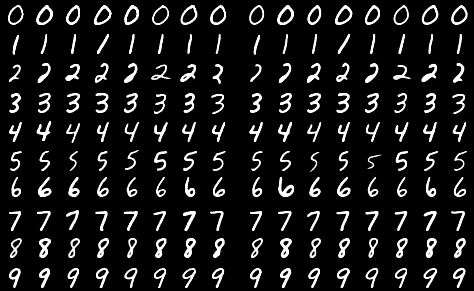}
    \captionof{figure}{
        Same-path prediction on MNIST.
        \textbf{Left}: SDE roll-outs from a fixed $x_0$ under eight
        Brownian trajectories.
        \textbf{Right}: four-step It\^o map predictions using the
        corresponding Brownian trajectories.
    }
    \label{mnpath}
\end{minipage}
\hfill
\begin{minipage}[t]{0.40\textwidth}
    \vspace{0pt}
    On MNIST, we test whether the It\^o map trained in pixel space can
    reproduce SDE transitions under matched Brownian trajectories.
    Figure~\ref{mnpath} compares SDE roll-outs with four-step It\^o map
    predictions from the same initial noise and Brownian paths. We used
    five KL coefficients when training the It\^o map. The pixel-space MSE
    is $0.0091$ on the $[0,1]$ scale for a four-step generation. The two- and one-step MSEs are
    $0.0143$ and $0.0191$, respectively. This supports the central premise
    that, conditioned on the same Brownian path, the learned any-step map
    approximates the stochastic roll-out.
\end{minipage}
\par
\subsection{MNIST Reward Alignment}\label{MNinfsteer}
We next evaluate inference-time steering on MNIST, where the goal is to steer generation toward a prescribed class mixture. We define the reward function as $r(x) \coloneqq \log \sum_{i=0}^9 w_i\,p_{\theta}(c_i\mid x)$ where $p_{\theta}(c_i\mid x)$ is the pretrained classifier probability that $x$ belongs to class $i$ and weights $w_i$ specifying the desired class proportions satisfy $\sum w_i=1$. Under this choice, the reward-tilted target corresponds to the class-mixture posterior $\sum_{i=0}^9w_i \,p(x\mid c_i)$, so success can be measured by how closely the generated class histogram matches the target weights.\vspace{-1em}
\begin{figure}[H]
    \centering
    \begin{minipage}[c]{0.5\linewidth}
        \includegraphics[width=\linewidth]{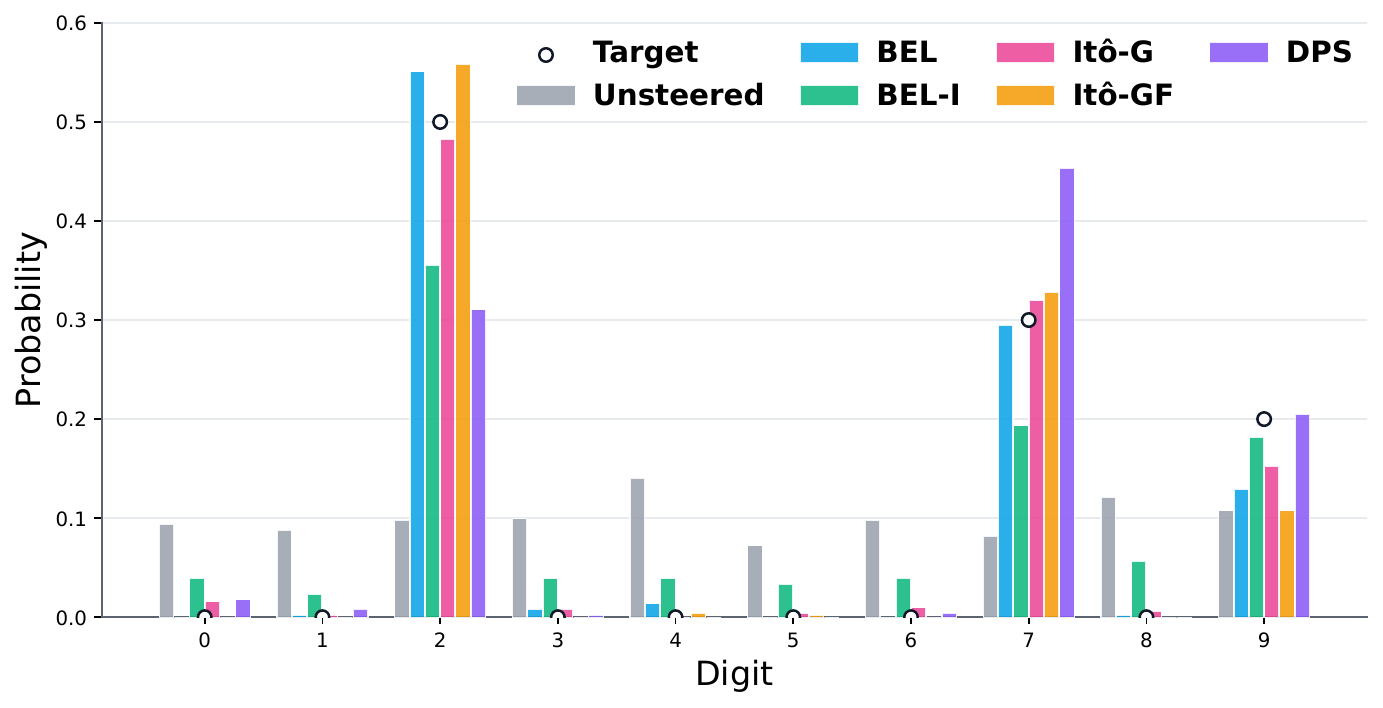}
    \end{minipage}\hfill
    \begin{minipage}[c]{0.45\linewidth}
        \centering\includegraphics[width=\linewidth]{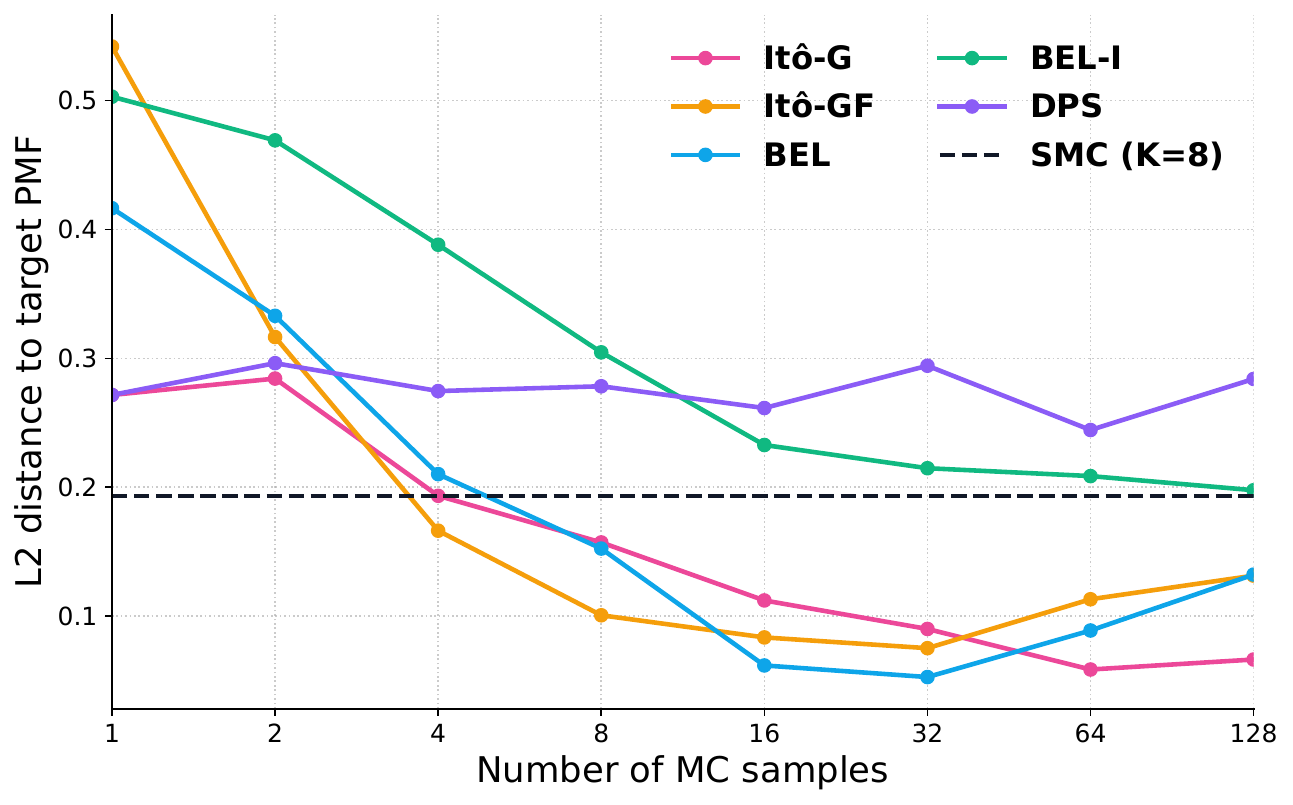}
        \renewcommand{\arraystretch}{1}
    \end{minipage}\vspace{-1em}
    \caption{MNIST reward alignment toward a target mixture supported on digits 2, 7, 9 with weights $(w_2,w_7,w_9)=(0.5,0.3,0.2)$. \textbf{Left}: terminal classes of the 512 tilted generations with different estimators, using 64 Monte Carlo samples. \textbf{Right}: 
    $L^2$(steered class frequency, target weights) against the number of MC samples. }  
    \label{mnist}
\end{figure}\vspace{-1em}
Figure \ref{mnist} shows that the It\^o map estimators match the target mixture more accurately than the unsteered roll-out and DPS. In particular, DPS overshoots digit 7 and undershoots digit 2, whereas It\^o-G, It\^o-GF, and BEL produce class proportions closer to the desired weights. Quantitatively, the It\^o map estimators also achieve lower $L^2 $ error to the target PMF than the SMC baseline given enough MC samples. This experiment shows that the It\^o map control estimators remain effective beyond analytic toy problems, in an image setting where the target is specified through a downstream classifier.
\vspace{-0.5em}
\subsection{ImageNet-256}\vspace{-0.5em}
For ImageNet-256, we use an SiT-XL backbone \citep{sit} and evaluate It\^o maps in two settings. First, we use a self-distilled It\^o map to test conditional diversity and same-path fidelity from fixed intermediate states; second, we use an It\^o map trained with a pretrained Decoupled MeanFlow model \citep{dmf} as the diagonal teacher for reward steering.
\paragraph{Posterior diversity and same-path fidelity.} Fix a time $t$ and a state $X_t$. We simulate multiple i.i.d. Brownian trajectories $W^j$ on $[t,1]$ and use our It\^o map trained from scratch to predict the final states $\hat{X}_{t,1}(X_t,W^j)$. The It\^o map should produce diverse endpoints while remaining aligned with the corresponding SDE roll-outs driven by the same Brownian paths. As $t$ increases, more features would have been determined at $X_t$. For improved sample quality, we compose the learned any-step map over four coarse intervals.
\begin{figure}[H]
    \centering
    \begin{minipage}[t]{0.1\textwidth}
        \centering
        \includegraphics[width=\linewidth]{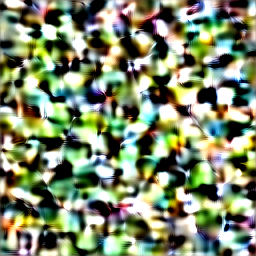}
        \small Fixed $X_0$
    \end{minipage}
    \hfill
    \begin{minipage}[t]{0.4\textwidth}
        \centering
        \includegraphics[width=\linewidth]{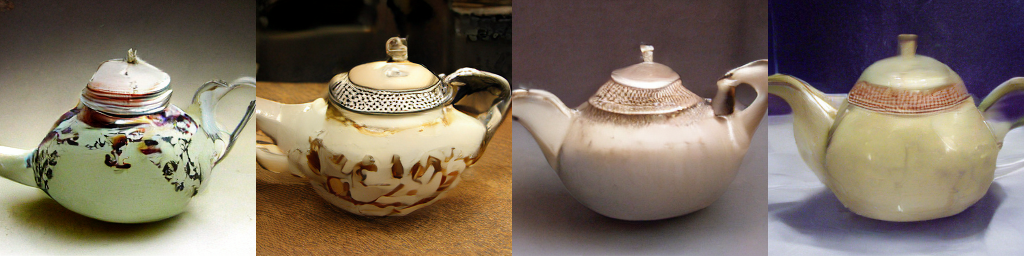}
        \small SDE roll-outs ($t=0$)
    \end{minipage}
    \hfill
    \begin{minipage}[t]{0.4\textwidth}
        \centering
        \includegraphics[width=\linewidth]{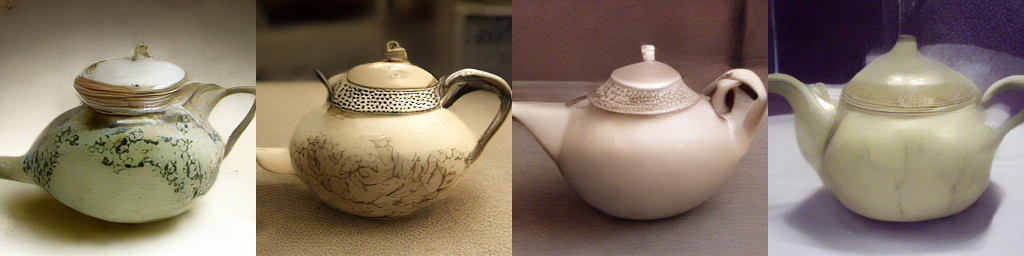}
        
        \small 4-step It\^o map predictions ($t=0$)
    \end{minipage}

    \vspace{6pt}

    \begin{minipage}[t]{0.1\textwidth}
        \centering
        \includegraphics[width=\linewidth]{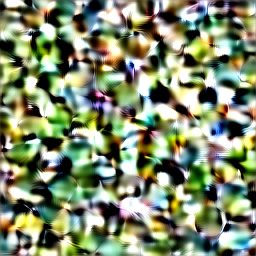}
        
        \small Fixed $X_{0.2}$
    \end{minipage}
    \hfill
    \begin{minipage}[t]{0.4\textwidth}
        \centering
        \includegraphics[width=\linewidth]{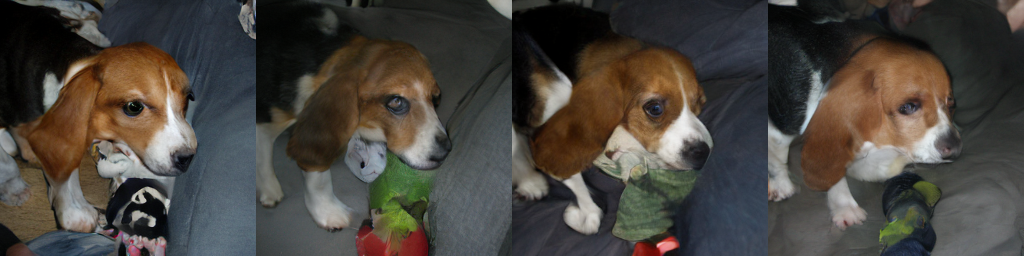}
        
        \small SDE roll-outs ($t=0.2$)
    \end{minipage}
    \hfill
    \begin{minipage}[t]{0.4\textwidth}
        \centering
        \includegraphics[width=\linewidth]{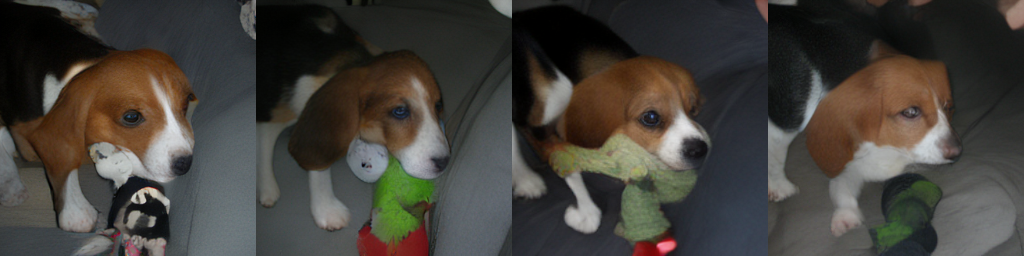}
        
        \small 4-step It\^o map predictions ($t=0.2$)
    \end{minipage}

    \caption{Posterior sampling from fixed intermediate $X_t$. \textbf{Top row}: sampling from fixed initial state $X_0$. \textbf{Bottom row}: sampling from fixed intermediate state $X_{0.2}$. In each row, \textbf{left}: fixed input state; \textbf{middle}: SDE roll-outs under sampled Brownian trajectories; \textbf{right}: It\^o map predictions using the same Brownian trajectories.}
    \label{fig:posterior_sampling_row}
\end{figure}
\paragraph{FID of the It\^o map generation.} We evaluate the It\^o map model by computing the FID of generated samples. For an It\^o map trained with SDVAE latents and deterministic DMF checkpoint teacher guidance, the FID scores of the 2-, 4- and 16-step It\^o map generations are reported as follows:\vspace{-0.5em}
\begin{table}[H]
\centering
\setlength{\tabcolsep}{4pt}
\begin{tabular}{@{}lrrrr@{}}
\toprule
Number of steps & 16 &  4 &2    \\
\midrule
FID & 1.41 & 2.10 & 6.43 \\
\bottomrule
\end{tabular}
\caption{FID for the It\^o map generation}
\label{tab:fid}
\end{table}\vspace{-2em}
\paragraph{Reward-tilted steering.} For reward-tilted steering, we use the It\^o map trained with the DMF diagonal teacher. This experiment evaluates the learned Brownian-conditioned off-diagonal map and the resulting control estimators on top of a strong pretrained diagonal drift. We compare the average reward of steered output using each estimator against the best-of-N, DPS, and MFM baselines using ImageReward \citep{imagereward} and HPSv2 \citep{hpsv2}. In Table \ref{tab:reward}, we report the average reward across 16 classes and prompts. All steering methods are evaluated on the same class-prompt set and use the same random initializations and Brownian trajectories when applicable; implementation details are given in Appendix \ref{table1append}.\vspace{-0.5em}
\begin{table}[H]
\centering
\setlength{\tabcolsep}{4pt}
\begin{tabular}{@{}lrrrrrrrrrr@{}}
\toprule
Reward type & SDE &  It\^o-GF &It\^o-G & BEL & BEL-I & DPS & Bo10 & Bo100 &  MFM-GF &MFM-G  \\
\midrule
ImageReward & 0.42 &  1.30 &\textbf{1.88} & 0.87 & 0.49 & 0.55 & 1.09 & 1.36  & 1.38& 1.83 \\
HPSv2       & 0.25 & 0.29 & \textbf{0.38} & 0.25 & 0.25 & 0.26 & 0.28 & 0.30 & 0.26 & 0.26 \\
\bottomrule
\end{tabular}
\caption{Average rewards for different estimators over classes and prompts.}
\label{tab:reward}
\end{table} \vspace{-2em}
Table \ref{tab:reward} shows that It\^o-G achieves the highest average reward under both reward models. On ImageReward, it slightly improves over MFM-G and outperforms DPS and Best-of-100; on HPSv2, it gives a larger gain over all baselines. It\^o-GF also improves over the unsteered SDE, DPS, and Best-of-10, and is competitive with Best-of-100 and MFM-GF. The BEL-family estimators provide smaller gains, with more consistent improvements on ImageReward than HPSv2. Overall, these results suggest that Brownian-conditioned It\^o maps support competitive large-scale reward steering, with the strongest performance obtained when reward gradients are available.

\paragraph{Steering with held-out metrics.}  To assess generalization beyond the steering reward, we steer using HPSv2 and evaluate the resulting samples with ImageReward. For $K$ Monte Carlo samples per step, Itô-G, Itô-GF, BEL-I, MFM-G, and MFM-GF each require $50(1+K)$ generator NFEs, matching the best-of-$(K+1)$ baseline. DPS uses 50 NFEs, while full BEL requires $50+K\sum_{i=0}^{49}(50-i)$ NFEs. Table 3 shows that our methods improve held-out reward over the unsteered SDE baseline.
\begin{table}[H]
\centering
\setlength{\tabcolsep}{4pt}
\resizebox{\linewidth}{!}{%
\begin{tabular}{@{}rrrrrrrrrr@{}}
\toprule
Ctrl particles $K$ & NFEs & SDE rollout & Bo$(K+1)$
& It\^o-GF & It\^o-G & BEL-I & DPS& MFM-G & MFM-GF \\
\midrule
8  & 450 & -0.086 & 0.66 & 0.87 & \textbf{1.15} & 0.071 & -0.04 & 0.95 & 0.82 \\
16 & 850 & -0.086 & 0.68 & 1.07 & \textbf{1.29} & 0.025 & -0.04 & 1.19 & 1.06 \\
\bottomrule
\end{tabular}}
\caption{Held-out average rewards for different estimators over classes and prompts. SDE-rollout rewards differ from table \ref{tab:reward} due to different random seeds.} 
\end{table}\vspace{-3em}
\section{Conclusion}
We introduced It\^o maps as Brownian-conditioned any-step transition operators for generative SDEs, together with distillation objectives for learning them.
In contrast to deterministic flow maps, the It\^o maps provide direct access to one-step transition kernels, explicitly parametrized by the Brownian path $W$. We use this structure to derive algorithms for inference-time steering, covering both settings where reward gradients are available and where they are unavailable or unstable. Across synthetic posterior sampling, MNIST class-mixture steering, and ImageNet reward steering, It\^o maps support diverse conditional sampling and competitive reward alignment.

We hope this work encourages further study of stochastic flow maps as primitives for posterior sampling and inference-time control.

\section*{Acknowledgments}
We thank Sam Howard, Iskander Azangulov, Adhi Saravanan, Shiyi Wang, Brian Lee, Carles Domingo-Enrich, and George Deligiannidis for fruitful conversations. PP is supported by the EPSRC CDT in Modern Statistics and Statistical Machine Learning [EP/S023151/1], a Google PhD Fellowship, and an NSERC Postgraduate Scholarship (PGS D). MSA is supported by a Junior Fellowship at the Harvard Society of Fellows as well as the National Science Foundation under Cooperative Agreement PHY-2019786 (The NSF AI Institute for Artificial Intelligence and Fundamental Interactions\footnote{http://iaifi.org/}). This work has been made possible in part by a gift from the Chan Zuckerberg Initiative Foundation to establish the Kempner Institute for the Study of Natural and Artificial Intelligence.

\bibliographystyle{assets/acl_natbib.bst}
\bibliography{references.bib}

@inproceedings{howtobuild,
  title={How to build a consistency model: Learning flow maps via self-distillation},
  author={Boffi, Nicholas M. and Albergo, Michael S. and Vanden-Eijnden, Eric},
  booktitle={Advances in Neural Information Processing Systems},
  year={2025}
}

@inproceedings{dmf,
  title = {Decoupled MeanFlow: Turning Flow Models into Flow Maps for Accelerated Sampling},
  author = {Lee, Kyungmin and Yu, Sihyun and Shin, Jinwoo},
  booktitle = {International Conference on Learning Representations},
  year = {2026}
}

@article{mfm,
  title = {Meta Flow Maps enable scalable reward alignment},
  author = {Potaptchik, Peter and Saravanan, Adhi and
            Mammadov, Abbas and Prat, Alvaro and
            Albergo, Michael S. and Teh, Yee Whye},
  journal = {arXiv preprint},
  year = {2026}
}

@misc{song2020score,
  title = {Score-Based Generative Modeling through Stochastic Differential Equations},
  author = {Song, Yang and Sohl-Dickstein, Jascha and Kingma, Diederik P. and Kumar, Abhishek and Ermon, Stefano and Poole, Ben},
  year = {2020},
  eprint = {2011.13456},
  archivePrefix = {arXiv},
  primaryClass = {cs.LG},
  url = {https://arxiv.org/abs/2011.13456}
}

@inproceedings{ddpm,
  title = {Denoising Diffusion Probabilistic Models},
  author = {Ho, Jonathan and Jain, Ajay N. and Abbeel, Pieter},
  booktitle = {Advances in Neural Information Processing Systems},
  volume = {33},
  pages = {6840--6851},
  year = {2020}
}

@inproceedings{lipman2023fm,
  title = {Flow Matching for Generative Modeling},
  author = {Lipman, Yaron and Chen, Ricky T. Q. and Ben-Hamu, Heli and Nickel, Maximilian and Le, Matthew},
  booktitle = {The Eleventh International Conference on Learning Representations},
  year = {2023}
}

@inproceedings{albergo23si,
  title = {Building Normalizing Flows with Stochastic Interpolants},
  author = {Albergo, Michael S. and Vanden-Eijnden, Eric},
  booktitle = {The Eleventh International Conference on Learning Representations},
  year = {2023}
}

@article{daipra,
  author = {Dai Pra, Paolo},
  title = {A Stochastic Control Approach to Reciprocal Diffusion Processes},
  journal = {Applied Mathematics and Optimization},
  volume = {23},
  number = {1},
  pages = {313--329},
  year = {1991},
  doi = {10.1007/BF01442404}
}

@article{kosambi,
  author = {Kosambi, D. D.},
  title = {Statistics in Function Space},
  journal = {Journal of the Indian Mathematical Society},
  volume = {7},
  pages = {76--88},
  year = {1943}
}

@article{karhunen,
  author = {Karhunen, Kari},
  title = {{\"U}ber lineare Methoden in der Wahrscheinlichkeitsrechnung},
  journal = {Annales Academiae Scientiarum Fennicae. Series A. I. Mathematica-Physica},
  volume = {37},
  year = {1947}
}

@incollection{loeve,
  author = {Lo\`eve, Michel},
  title = {Fonctions al\'eatoires du second ordre},
  booktitle = {Processus stochastiques et mouvement Brownien},
  publisher = {Gauthier-Villars},
  address = {Paris},
  year = {1948}
}

@InProceedings{BEL,
  title = 	 {Conditioning Diffusions Using Malliavin Calculus},
  author =       {Pidstrigach, Jakiw and Baker, Elizabeth Louise and Domingo-Enrich, Carles and Deligiannidis, George and N\"{u}sken, Nikolas},
  booktitle = 	 {Proceedings of the 42nd International Conference on Machine Learning},
  pages = 	 {49292--49315},
  year = 	 {2025},
  volume = 	 {267},
  series = 	 {Proceedings of Machine Learning Research},
  month = 	 {13--19 Jul},
  publisher =    {PMLR},
  url = 	 {https://proceedings.mlr.press/v267/pidstrigach25a.html}
}

@inproceedings{consistency,
  title = {Consistency Models},
  author = {Song, Yang and Dhariwal, Prafulla and Chen, Mark and Sutskever, Ilya},
  booktitle = {Proceedings of the 40th International Conference on Machine Learning},
  series = {Proceedings of Machine Learning Research},
  volume = {202},
  pages = {32211--32252},
  year = {2023}
}

@inproceedings{dps,
  title = {Diffusion Posterior Sampling for General Noisy Inverse Problems},
  author = {Chung, Hyungjin and Kim, Jeongsol and McCann, Michael T. and Klasky, Marc L. and Ye, Jong Chul},
  booktitle = {The Eleventh International Conference on Learning Representations},
  year = {2023}
}

@inproceedings{dit,
  title = {Scalable Diffusion Models with Transformers},
  author = {Peebles, William and Xie, Saining},
  booktitle = {Proceedings of the IEEE/CVF International Conference on Computer Vision},
  pages = {4195--4205},
  year = {2023}
}

@inproceedings{sit,
  title = {SiT: Exploring Flow and Diffusion-Based Generative Models with Scalable Interpolant Transformers},
  author = {Ma, Nanye and Goldstein, Mark and Albergo, Michael S. and Boffi, Nicholas M. and Vanden-Eijnden, Eric and Xie, Saining},
  booktitle = {European Conference on Computer Vision},
  year = {2024}
}

@inproceedings{imagereward,
  title = {ImageReward: Learning and Evaluating Human Preferences for Text-to-Image Generation},
  author = {Xu, Jiazheng and Liu, Xiao and Wu, Yuchen and Tong, Yuxuan and Li, Qinkai and Ding, Ming and Tang, Jie and Dong, Yuxiao},
  booktitle = {Advances in Neural Information Processing Systems},
  volume = {36},
  year = {2023}
}

@article{hpsv2,
  title={Human Preference Score v2: A Solid Benchmark for Evaluating Human Preferences of Text-to-Image Synthesis},
  author={Wu, Xiaoshi and Hao, Yiming and Sun, Keqiang and Chen, Yixiong and Zhu, Feng and Zhao, Rui and Li, Hongsheng},
  journal={arXiv preprint arXiv:2306.09341},
  year={2023}
}

@misc{tvm,
  title = {Terminal Velocity Matching},
  author = {Zhou, Linqi and Parger, Mathias and Haque, Ayaan and Song, Jiaming},
  year = {2025},
  eprint = {2511.19797},
  archivePrefix = {arXiv},
  primaryClass = {cs.LG},
  url = {https://arxiv.org/abs/2511.19797}
}

@article{abve_jmlr,
  author  = {Albergo, Michael and Boffi, Nicholas M. and Vanden-Eijnden, Eric},
  title   = {Stochastic Interpolants: A Unifying Framework for Flows and Diffusions},
  journal = {Journal of Machine Learning Research},
  year    = {2025},
  volume  = {26},
  number  = {209},
  pages   = {1--80},
  url     = {http://jmlr.org/papers/v26/23-1605.html}
}

@article{lagrangian,
  author  = {Arnold, Vladimir I.},
  title   = {Sur la g\'eom\'etrie diff\'erentielle des groupes de {L}ie de dimension infinie et ses applications \`a l'hydrodynamique des fluides parfaits},
  journal = {Annales de l'Institut Fourier},
  volume  = {16},
  number  = {1},
  pages   = {319--361},
  year    = {1966},
  doi     = {10.5802/aif.233},
  url     = {https://aif.centre-mersenne.org/articles/10.5802/aif.233/}
}

@inproceedings{rectifiedflow,
  title = {Flow Straight and Fast: Learning to Generate and Transfer Data with Rectified Flow},
  author = {Liu, Xingchao and Gong, Chengyue and Liu, Qiang},
  booktitle = {The Eleventh International Conference on Learning Representations},
  year = {2023}
}

@article{doob1957,
  author  = {Doob, J. L.},
  title   = {Conditional Brownian Motion and the Boundary Limits of Harmonic Functions},
  journal = {Bulletin de la Soci{\'e}t{\'e} Math{\'e}matique de France},
  volume  = {85},
  pages   = {431--458},
  year    = {1957},
  doi     = {10.24033/bsmf.1494}
}

@inproceedings{freedom,
  title     = {FreeDoM: Training-Free Energy-Guided Conditional Diffusion Model},
  author    = {Yu, Jiwen and Wang, Yinhuai and Zhao, Chen and Ghanem, Bernard and Zhang, Jian},
  booktitle = {Proceedings of the IEEE/CVF International Conference on Computer Vision},
  year      = {2023}
}

@inproceedings{ugd,
  title     = {Universal Guidance for Diffusion Models},
  author    = {Bansal, Arpit and Chu, Hong-Min and Schwarzschild, Avi and Sengupta, Soumyadip and Goldblum, Micah and Geiping, Jonas and Goldstein, Tom},
  booktitle = {International Conference on Learning Representations},
  year      = {2024}
}

@inproceedings{lgd,
  title     = {Loss-Guided Diffusion Models for Plug-and-Play Controllable Generation},
  author    = {Song, Jiaming and Zhang, Qinsheng and Yin, Hongxu and Mardani, Morteza and Liu, Ming-Yu and Kautz, Jan and Chen, Yongxin and Vahdat, Arash},
  booktitle = {Proceedings of the 40th International Conference on Machine Learning},
  series    = {Proceedings of Machine Learning Research},
  volume    = {202},
  pages     = {32483--32498},
  year      = {2023}
}

@misc{uehara2025ita,
  title         = {Inference-Time Alignment in Diffusion Models with Reward-Guided Generation: Tutorial and Review},
  author        = {Uehara, Masatoshi and Zhao, Yulai and Wang, Chenyu and Li, Xiner and Regev, Aviv and Levine, Sergey and Biancalani, Tommaso},
  year          = {2025},
  eprint        = {2501.09685},
  archivePrefix = {arXiv},
  primaryClass  = {cs.LG},
  url           = {https://arxiv.org/abs/2501.09685}
}

@inproceedings{meanflow,
  title     = {Mean Flows for One-step Generative Modeling},
  author    = {Geng, Zhengyang and Deng, Mingyang and Bai, Xingjian and Kolter, J. Zico and He, Kaiming},
  booktitle = {Advances in Neural Information Processing Systems},
  year      = {2025}
}

@misc{mpgd,
  title         = {Manifold Preserving Guided Diffusion},
  author        = {He, Yutong and Murata, Naoki and Lai, Chieh-Hsin and Takida, Yuhta and Uesaka, Toshimitsu and Kim, Dongjun and Liao, Wei-Hsiang and Mitsufuji, Yuki and Kolter, J. Zico and Salakhutdinov, Ruslan and Ermon, Stefano},
  year          = {2023},
  eprint        = {2311.16424},
  archivePrefix = {arXiv},
  primaryClass  = {cs.LG},
  url           = {https://arxiv.org/abs/2311.16424}
}

@inproceedings{vargas2022nsf,
  title     = {Bayesian Learning via Neural Schr{\"o}dinger-F{\"o}llmer Flows},
  author    = {Vargas, Francisco and Ovsianas, Andrius and Fernandes, David Lopes and Girolami, Mark and Lawrence, Neil D. and N{\"u}sken, Nikolas},
  booktitle = {Symposium on Advances in Approximate Bayesian Inference},
  year      = {2022},
  url       = {https://openreview.net/forum?id=1Fqd10N5yTF}
}

@inproceedings{akhoundsadegh2024idem,
  title     = {Iterated Denoising Energy Matching for Sampling from Boltzmann Densities},
  author    = {Akhound-Sadegh, Tara and Rector-Brooks, Jarrid and Bose, Joey and Mittal, Sarthak and Lemos, Pablo and Liu, Cheng-Hao and Sendera, Marcin and Ravanbakhsh, Siamak and Gidel, Gauthier and Bengio, Yoshua and Malkin, Nikolay and Tong, Alexander},
  booktitle = {Proceedings of the 41st International Conference on Machine Learning},
  series    = {Proceedings of Machine Learning Research},
  volume    = {235},
  pages     = {760--786},
  year      = {2024},
  url       = {https://proceedings.mlr.press/v235/akhound-sadegh24a.html}
}

@inproceedings{ctm,
  title     = {Consistency Trajectory Models: Learning Probability Flow {ODE} Trajectory of Diffusion},
  author    = {Kim, Dongjun and Lai, Chieh-Hsin and Liao, Wei-Hsiang and Murata, Naoki and Takida, Yuhta and Uesaka, Toshimitsu and He, Yutong and Mitsufuji, Yuki and Ermon, Stefano},
  booktitle = {International Conference on Learning Representations},
  year      = {2024}
}

@inproceedings{shortcut,
  title     = {One Step Diffusion via Shortcut Models},
  author    = {Frans, Kevin and Hafner, Danijar and Levine, Sergey and Abbeel, Pieter},
  booktitle = {International Conference on Learning Representations},
  year      = {2025}
}

@misc{sebastian_mnist_classifier,
  author = {Sebastian, Jeril},
  title = {{jerilseb/mnist-classifier}},
  howpublished = {Hugging Face model repository},
  year = {2026},
  note = {Accessed: 2026}
}

@book{kunita1990stochastic,
  title={Stochastic flows and stochastic differential equations},
  author={Kunita, Hiroshi},
  volume={24},
  year={1990},
  publisher={Cambridge university press}
}

\clearpage

\section{Mathematical Derivations and Proofs}\label{appenda}
\paragraph{Proof to (\ref{eq:sdeinterpdrift}).}\label{proofdiag} Recall that the PF-ODE has diagonal 
\begin{equation}
    b_t(x)= \mathbb{E}[X_1-X_0\mid I_t=x].
\end{equation} By standard Tweedie identity for Gaussian-corrupted interpolants (see, e.g. \citep{abve_jmlr}), we also know 
\begin{equation}
    \nabla\log p_t(x)=\mathbb{E}[-X_0/(1-t)\mid I_t=x].
\end{equation} We then deduce that \begin{equation}
    b_t(x)+\frac{\sigma_t^2}{2}\nabla\log p_t(x)=\mathbb{E}_{(X_0,X_1)}\left[ X_1-\left(1+\frac{\sigma_t^2}{2(1-t)}\right)X_0 \mid I_t=x\right].\qed 
\end{equation}
\subsection{Consistency Loss Objectives}
In this section, we first introduce the PSD objective, and then show that
\begin{itemize}
    \item [(1)] LSD and PSD are both valid consistency objectives for It\^o map training. Indeed, the unique zero minimizer of either of these loss objectives is the solution to the generative SDE;
    \item [(2)] On the other hand, we prove Eulerian-type loss objectives cannot be used to train It\^o maps.
\end{itemize}
\paragraph{Progressive self-distillation (PSD) objective.} One important property of the It\^o map is its semigroup property: $\Phi_{s,t}(\cdot, W) = \Phi_{u,t}(\cdot, W) \circ \Phi_{s, u}(\cdot, W)$. Therefore we can augment the learning of stochastic flow maps with the following objective:
\begin{equation}
    \mathcal{L}_{\text{PSD}}(F) = \int_{s \leq u \leq t} \mathbb E \left[\left \|F_{s, t}( X_s, W) - F_{u, t}(F_{s, u}(X_s, W), W)\right \|^2\right ]dsdudt. \label{psd}
\end{equation}
Just as LSD, this objective can also be simplified to depend only on the average velocity:
\begin{prop} (c.f. \citep{howtobuild})
    The PSD loss (\ref{psd}) is equivalent to a loss function involving only the Lipschitz component $G_{s,t}$, where $\gamma = \frac{u-s}{t-s}$ and $1 - \gamma = \frac{t-u}{t-s}$:
    \begin{align}
        \L_{\PSD}(F) = \int\limits_{s \leq u \leq t} \mathbb E \left[\left \|G_{s, t}( I_s, W) - \gamma \cdot G_{s,u}(I_s, W) - (1-\gamma) G_{u,t} (F_{s,u}(I_s), W) \right \|^2\right ]dsdudt.
    \end{align}
    where we WLOG omit $(t-s)^2$ in the integrand.
\end{prop}
\begin{proof}
Substituting the definition of $F_{s,t}$ into both terms of the PSD loss:
\begin{align}
F_{s,t}(X_s, W) =& X_s + (t-s)G_{s,t}(X_s, W) + (M_t - M_s) \\
F_{u,t}(F_{s,u}(X_s, W), W) =& X_s + (u-s)G_{s,u}(X_s,W) + (t-u)G_{u,t}(F_{s,u}(X_s,W),W) \\
&+ (M_u - M_s) + (M_t - M_u).
\end{align}
Taking the difference, $X_s$ cancels and the martingale terms sum to zero, so we are left with
\[
(t-s)\Bigl[G_{s,t}(I_s,W) - \gamma\, G_{s,u}(I_s,W) - (1-\gamma)\,G_{u,t}(F_{s,u}(I_s),W)\Bigr],
\]
where $\gamma = \frac{u-s}{t-s}$ and $1-\gamma = \frac{t-u}{t-s}$. Dividing out the positive factor \(t-s\) gives a rescaled PSD objective with the same zero set.
\end{proof}
Hence, the PSD training algorithm is given as in Algorithm \ref{alg:PSD}.
  \begin{algorithm}[H]
\caption{Training from scratch with progressive self-distillation (PSD).}
\label{alg:PSD}
\begin{algorithmic}[1]
  \Input $p_0 = p_{\text{noise}},\; p_1 = p_{\text{data}}$; model $\hat{G}_{s,t}(x,W)$
  \Repeat
    \State \textbf{Compute diagonal loss:}
    \State Sample time $t \sim U[0,1]$; simulate $X_0 \sim p_0$ and $X_1 \sim p_{\text{data}}$
    \State $I_t \gets tX_1 + (1-t)X_0$
    \State $\mathcal{L}_{\text{SI}} \gets \|\hat{G}_{t,t}(I_t,\text{zeros}) - (X_1 - 2X_0)\|_2^2$

    \State \textbf{Compute consistency loss:}
    \State Sample $(s,u,t)\sim U[0,1]$ and reorder such that $s<u<t$
    \State Sample $X_0\sim p_0$ and $X_1\sim p_{\text{data}}$
    \State Simulate Brownian trajectory $(W_l)_{l\in [0,1]}$ or $(M_l)_{l\in [s,1]}$ and extract Brownian features $W$
    \State $I_s \gets sX_1 + (1-s)X_0$
    \State $X_{u\mid s} \gets I_s + (u-s)\hat{G}_{s,u}(I_s,W) + M_u - M_s$
    \State $\gamma \gets \frac{u-s}{t-s}$; project the KL coefs on $[s,1]$ to $[u,1]$, denoted as $ W_{[u,1]}$ 
    \State $\mathcal{L}_{\text{PSD}} \gets \left\|\hat{G}_{s,t}(I_s,W) - \text{stopgrad}\left[\gamma\,\hat{G}_{s,u}(I_s,W) + (1-\gamma)\hat{G}_{u,t}(X_{u\mid s},W_{[u,1]})\right]\right\|_2^2$

    \State $\mathrm{loss} \gets \mathcal{L}_{\text{SI}} + \lambda \cdot \mathcal{L}_{\text{PSD}}$
    \State Update model by taking one optimizer step
  \Until{convergence}
  \Output Trained It\^o map $ \hat{X}_{s,t}(x,W) = x + (t-s)\hat{G}_{s,t}(x,W) + M_t - M_s$
\end{algorithmic}
\end{algorithm}

Next, we show that the It\^o map is indeed the unique minimizer of the LSD and PSD objectives.
\begin{theorem}
    Assume that $\Phi_{s,t}(x,W)$ is the flow map of the stochastic dynamics\begin{equation}
        dX_t=b_t(X_t)dt+\sigma_t \,dW_t \label{stochasticdynamics}
    \end{equation}
    and \begin{equation}
        M_t=\int_0^t\sigma_udW_u
    \end{equation}
    with the average velocity $v_{s,t}(x,W)=(\Phi_{s,t}(x,W)-x-M_t+M_s)/(t-s)$ continuous in both time arguments, and $v_{t,t}(x,W)=b_t(x)$. Then a velocity field $G_{s,t}(x,W)$ (with $G$ continuous in all arguments, $G_{t,t}=b_t$) is identically equal to $v_{s,t}(x,W)$ if and only if one of the following conditions hold:
    
    \begin{enumerate}[label=(\roman*)]
        \item (Lagrangian condition): $t\mapsto (t-s)G_{s,t}$ is differentiable and $G_{s,t}$ solves the Lagrangian equation
        \begin{align}
           G_{t,t}(x+(t-s)G_{s,t}(x,W)+M_t-M_s,W)=G_{s,t}(x,W)+(t-s)\partial_tG_{s,t}(x,W) \label{loss:lsd}
        \end{align}
        for all $s\leq t$ in $[0,1]$ and for all $x \in \R^d$.

        \item (Semigroup condition): for all $s\leq u\leq t$ in $[0,1]$ and $x \in \R^d$ and $\gamma=(u-s)/(t-s)$: 
        \begin{align}
            G_{s,t}(x,W)=\gamma G_{s,u}(x,W)+(1-\gamma)G_{u,t}(x+(u-s)G_{s,u}(x,W)+M_u-M_s,W) \label{loss:psd}
        \end{align}
    \end{enumerate} \label{thm:cond}
\end{theorem}
\begin{proof}
    It is immediate that $v_{s,t}$ satisfies (\ref{loss:lsd}) due to the general theory of stochastic flows. For uniqueness, let's assume that there exists another velocity field ${G}_{s,t}(x,W)$ satisfying (\ref{loss:lsd}), i.e. for any $s\leq t,x\in\mathbb{R}^d$ and Brownian representation $W$:
    \begin{equation}
        {G}_{s,t}(x,W)+(t-s)\partial_t {G}_{s,t}(x,W)= {G}_{t,t}(x+(t-s){G}_{s,t}(x,W)+M_t-M_s,W).
    \end{equation}
    Now define the process\begin{equation}
        F_{s,t}(x,W)\coloneqq x+(t-s){G}_{s,t}(x,W)+M_t-M_s \label{flowF}
    \end{equation} and $A_{s,t}(x,W)\coloneqq (t-s){G}_{s,t}(x,W)$, the equation above and the diagonal condition implies
    \begin{equation}
        \partial_tA_{s,t}(x,W)=b_t(F_{s,t}(x,W)).
    \end{equation}
    Integrating both sides from $s$ to $t$ yields\begin{align}
F_{s,t}(x,W)&=x+\int_s^t\partial_uA_{s,u}(x,W)du+M_t-M_s\\
&=x+\int_s^tb_u(F_{s,u}(x,W))du+M_t-M_s.
    \end{align}
    But then $F_{s,t}$ is by definition the It\^o map for the stochastic dynamics (\ref{stochasticdynamics}), and hence $G_{s,t}\equiv v_{s,t}$.\medskip\\
    For the semigroup condition, it is again clear that the average velocity of the true It\^o map must satisfy (\ref{loss:psd}) by the general theory of stochastic flows, and we assume that ${G}$ is another velocity field satisfying (\ref{loss:psd}), i.e. for any $s\leq u\leq  t,x\in\mathbb{R}^d,W$, we have\begin{equation}
        {G}_{s,t}(x,W)=\gamma \,{G}_{s,u}(x,W)+(1-\gamma){G}_{u,t}(x+(u-s){G}_{s,u}(x,W)+M_u-M_s,W)
    \end{equation}
    Define $F_{s,t}(x,W) \coloneqq x + (t-s)G_{s,t}(x,W)+M_t-M_s$, then the condition above translates to \begin{equation}
        F_{s,t}(x,W)=F_{u,t}(F_{s,u}(x,W),W).
    \end{equation}
Define also $X_t := F_{s,t}(x, W)$, and fix a partition
\begin{equation}
    s = t_0 < t_1 < \cdots < t_n < t_{n+1} = t
\end{equation} of $[s,t]$.
Applying the semigroup condition repeatedly and substituting the parameterization
$F_{s,t}(x,W) = x + (t-s)G_{s,t}(x,W) + (M_t - M_s)$ for arbitrary $s\leq t$, we get
\begin{align*}
X_t &= F_{t_n, t}(X_{t_n}, W) \\
    &= X_{t_n} + (t - t_n)G_{t_n, t}(X_{t_n}, W) + (M_t - M_{t_n}) \\
    &= F_{t_{n-1},t_n}(X_{t_{n-1}}, W)
       + (t-t_n)G_{t_n,t}(X_{t_n},W) + (M_t - M_{t_n}) \\
    &= X_{t_{n-1}}
       + (t_n - t_{n-1})G_{t_{n-1},t_n}(X_{t_{n-1}},W)
       + (M_{t_n} - M_{t_{n-1}}) + (t-t_n)G_{t_n,t}(X_{t_n},W) + (M_t - M_{t_n}).
\end{align*}
Telescoping over the full partition, the martingale increments sum to $M_t - M_s$
and the $X$ terms collapse to $X_{t_0}=X_s=x$, giving
\[
X_t = x + \sum_{i=0}^{n} (t_{i+1} - t_i)\, G_{t_i, t_{i+1}}(X_{t_i}, W) + M_t-M_s.
\]
As the mesh $\max_i(t_{i+1}-t_i) \to 0$, continuity of $G_{s,t}$ in both arguments
gives $G_{t_i, t_{i+1}}(X_{t_i}, W) \to G_{t_i,t_i}(X_{t_i})$, so the Riemann sum converges to
$\int_s^t G_{u,u}(X_u)\,du$. Therefore
\[
X_t = x + \int_s^t G_{u,u}(X_u)\,du + M_t-M_s.
\]
If $G_{u,u}$ satisfies the diagonal condition
$G_{u,u}(x) = b_u(x)$, then $F_{s,t}(x,W)$ is indeed the It\^o map.
\end{proof}
We now apply Theorem \ref{thm:cond} above for general stochastic dynamics to our specific generative setup:
\begin{cor} Assume that the training distribution has full support over the relevant $(s,t,x,W)$ domain and $G$ is continuous. With the notation in (\ref{eq:gensde}) and assume that the diagonal training has converged, i.e. the diagonal $G_{t,t}$ satisfies 
    \begin{align}
        G_{t,t}(x,W)=b_t(x)+\frac{\sigma_t^2}{2}\nabla\log p_t(x)
    \end{align}
    for every $t\in[0,1],x\in\mathbb{R}^d$ and every Brownian representation $W$, then $G_{s,t}$ is the average velocity of the true It\^o map if and only if it is the unique minimizer over $\Ghat_{s,t}$ of any of the following two objectives:
    \begin{enumerate}[label=(\roman*)]
        \item The Lagrangian self-distillation (LSD) objective:
        \begin{align}
            \mathcal{L}_{\text{LSD}}(\Ghat) = \int_{s \leq t} \mathbb E \left [\|\Ghat_{t,t}(X_t, W) - \Ghat_{s,t}(I_s, W) - (t-s) \partial_t \Ghat_{s,t}(I_s, W) \|^2\right]dsdt
        \end{align}

        \item The progressive self-distillation (PSD) objective where $\gamma = \frac{u-s}{t-s}$:
        \begin{align}
            \mathcal{L}_{\text{PSD}}(\Ghat) = \int_{s \leq u\leq t} \mathbb E \left [\| \Ghat_{s,t} (I_s, W) - \big( \gamma \cdot \Ghat_{s,u}(I_s, W) + (1- \gamma) \cdot \Ghat_{u,t}(X_u,W) \big)  \|^2\right]dsdudt
        \end{align}
        Above, the expectation $\mathbb{E}[\cdot]$ is taken over the random draws $(x_0, x_1)$, $I_s(x_0,x_1) \coloneqq (1-s)x_0 + sx_1$ and $X_t=I_s+(t-s)\hat{G}_{s,t}(I_s,W)+M_t-M_s$ for $t\geq s$, with $M_t=\int_0^t \sigma_u dW_u$.
        \end{enumerate}
\end{cor}
\begin{proof}
    Let $B_t(x)=b_t(x)+\frac{\sigma_t^2}{2}\nabla\log p_t(x)$. Convergence of diagonal implies that $G_{t,t}(x,W)=B_t(x)$. Let $v_{s,t}(x,W)$ be the average velocity of the true It\^o map, then by Theorem \ref{thm:cond}, $v_{s,t}$ satisfies both Lagrangian and semigroup conditions. This in particular implies that 
   \begin{equation}
       \mathcal{L}_{\text{LSD}}(v)=\mathcal{L}_{\text{PSD}}(v)=0.
   \end{equation} 
   Now let $G_{s,t}$ be a global minimizer. Since the loss is an integral of a squared norm and its minimum value is 0, the residual must vanish almost surely. This means that the Lagrangian condition (\ref{loss:lsd}) and the semigroup condition (\ref{loss:psd}) hold for $(s,t,x,W)$ almost surely. With training measure having full support on the relevant domain, and $G_{s,t}(x,W)$ being continuous in all arguments, (\ref{loss:lsd}) and (\ref{loss:psd}) hold for any $(s,t,x,W)$. By Theorem \ref{thm:cond}, $G_{s,t}$ is the average velocity of the true It\^o map.
\end{proof}

\paragraph{Discussion on Eulerian self-distillation (ESD).} \label{esd}
The Eulerian perspective asks: given that my evolution will pass through $X_t$ at time $t$, what must the dynamics at time $s$ look like? Using the path-wise It\^o formula to differentiate $F_{s,t}$ with respect to the semi-martingale $(s, X_s)$, after algebraic manipulation we obtain (where $\partial_s$ means differentiating through all $s$-dependence, $x=y+M_s$) \begin{equation}
    G_{s,t}(x,W)=(t-s)\partial_sG_{s,t}(x,W)+(I+(t-s)\nabla_xG_{s,t}(x,W))b_s(x)
\end{equation}
In practice, autodiff treats the state input \(x\) and sampled Brownian features \(W\) as constants when differentiating with respect to \(s\). Intuitively, it is clear that ESD is not the right objective to use because the only term carrying the real Brownian displacement, \begin{equation}
     X_t=I_s+(t-s)G_{s,t}(I_s,W)+M_t-M_s
 \end{equation} is not present in this condition. Consequently, the model would be confused about how to interpret the KL coefficients as Brownian motions because the objective provides no pathwise constraint linking the Brownian features to the Brownian displacement. Indeed, the following proposition shows that this practical objective admits a zero-loss solution independent of the Brownian input.
\begin{prop}
    The naive generalization of deterministic Eulerian objective, defined as\begin{equation}
        \mathcal{L}_{\text{ESD}}(\Ghat)=\int_{s\leq t}\mathbb{E}\left\|\Ghat_{s,t}(x,W)-(t-s)\partial_s\Ghat_{s,t}(x,W)-(I+(t-s)\nabla_x\Ghat_{s,t}(x,W))b_s(x)\right\|^2dsdt\label{loss:eur}
    \end{equation} admits a zero minimizer satisfying the diagonal condition but is independent of the Brownian input $W$.
\end{prop}
\begin{proof}
    Consider the flow $\psi_{s,t}(x)$ of ODE \begin{equation}
        dX_t=b_t(X_t)dt
    \end{equation}
    and define
    \begin{equation}
        u_{s,t}(x) = \frac{\psi_{s,t}(x)-x}{t-s}
    \end{equation}
    Using the semigroup property of the flow $\psi$, we have \begin{equation}
        \psi_{s,t}(x)=\psi_{u,t}(\psi_{s,u}(x))
    \end{equation}
    Differentiating both sides w.r.t. $u$ and setting $u=s$, we get
    \begin{equation}
        0=\partial_s\psi_{s,t}(x)+\nabla_x\psi_{s,t}(x)\cdot b_s(x)
    \end{equation}
    Expanding with $\psi_{s,t}(x)=x+(t-s)u_{s,t}(x)$ and rearranging, we have:
    \begin{equation}
        u_{s,t}(x)-(t-s)\partial_su_{s,t}(x)-(I+(t-s)\nabla_x u_{s,t}(x))b_s(x)=0
    \end{equation}
    namely $u_{s,t}$ also has zero loss for (\ref{loss:eur}). This concludes the proof. 
\end{proof}
\color{black}
\subsection{Choice of Brownian Representation}
We explain why our Brownian representation is independent of the endpoint $t$. Writing $F_{s,t}$ as
\begin{equation}
    F_{s,t}(x,W)
    = x+A_{s,t}(x,W)+M_t-M_s,
    \quad
    A_{s,t}(x,W)=(t-s)G_{s,t}(x,W),
    \quad
    M_t=\int_0^t \sigma_u\,dW_u,
\end{equation}
reduces the SDE~(\ref{stochasticdynamics}), for almost every fixed
Brownian realization, to the random ODE
\begin{equation}
    \frac{d}{dt}A_{s,t}(x,W)
    = b_t\bigl(x+A_{s,t}(x,W)+M_t-M_s\bigr),
    \qquad A_{s,s}(x,W)=0.
\end{equation}
This is the pathwise argument used in the proof of
Theorem~\ref{thm:cond}. We
differentiate the drift component $A_{s,t}$ rather than the
Brownian path. Substituting $A_{s,t}=(t-s)G_{s,t}$ and using the diagonal condition
$G_{t,t}=b_t$, Lagrangian self-consistency therefore requires
\begin{equation}
    G_{s,t}(x,W)
    +(t-s)\frac{d}{dt}G_{s,t}(x,W)
    =G_{t,t}\bigl(F_{s,t}(x,W),W\bigr),
\end{equation}
which is precisely~(\ref{loss:lsd}), with the derivative understood
along the same Brownian realization. Crucially, this must be the
\emph{total derivative} with respect to $t$, holding $s$, $x$, and
the underlying path fixed. If the network receives a representation
$\zeta_{s,t}(W)$ that depends on $t$, the required derivative is
\begin{equation}
    \frac{d}{dt}G\bigl(s,t,x,\zeta_{s,t}(W)\bigr), \label{totder}
\end{equation}
including all dependence through that representation whenever
this derivative exists. Taking only the partial derivative with
the representation held fixed generally gives a different
condition; in particular, a time-dependent martingale increment
cannot simply be treated as a detached input.

Alternatively, because the random-ODE identity holds for almost
every fixed path, we may supply the full Brownian path on $[0,1]$,
or the full future path on $[s,1]$, using coordinates independent
of the endpoint $t$. With $s$ fixed, holding these coordinates
constant then makes the partial derivative with respect to $t$
equal to the required total derivative. This motivates our choice
of KL coefficients defined on $[s,1]$ rather than on $[s,t]$.

\textit{Remark for PSD.} Although the PSD objective contains no explicit time derivatives, the total derivative in (\ref{totder}) must still exist for $t>s$ if the learned map represents the flow of (\ref{stochasticdynamics}). Without this pathwise regularity, semigroup consistency admits incorrect solutions. To illustrate this, let $\psi_{s,t}$ denote the deterministic flow generated by $b_t$, and define
\begin{equation}
    \hat{G}_{s,t}(x,W)\coloneqq (\psi_{s,t}(x)-x-(M_t-M_s))/(t-s) \;\text{for }s<t\quad\text{and}\quad \hat{G}_{t,t}(x,W)={b}_t(x)
\end{equation}
The reconstructed map $\hat{F}_{s,t}(x,W)=\psi_{s,t}(x)$ has the correct drift and zero PSD residual, yet no Brownian dependence. This construction is pathological in continuous time: $\widehat G_{s,t}$ has no pathwise time derivative and is discontinuous at the diagonal, violating the continuity assumption in Theorem \ref{thm:cond}. However, when $M_t-M_s$ is supplied as an input feature, a network can learn this cancellation on the sampled intervals while remaining smooth in its explicit inputs. Such smoothness does not guarantee differentiability along the Brownian realization. Moreover, on the training grid, nonzero intervals satisfy $t-s\ge 1/N$, so the singular diagonal limit is never tested. Empirically, supplying \(M_t-M_s\) directly yielded low consistency loss but incorrect samples, consistent with shortcut learning. We therefore omit this explicit feature.

\subsection{Karhunen–Loève Expansion}\label{klthm}
For completeness of the paper, we provide a proof for the commonly known theorem on the KL expansion:
\begin{theorem}
Let $(W_t)_{t\in[0,T]}$ be a standard Brownian motion defined on a probability sample space $\Omega$. Then there exist
independent standard Gaussian random variables $(\xi_n)_{n\ge 1}$ such that
for every $t\in[0,T]$,
\[
W_t
=
\sum_{n=1}^{\infty}
\frac{T}{(n-\tfrac12)\pi}\,\xi_n\,
\sqrt{\frac{2}{T}}
\sin\!\Big(\frac{(n-\tfrac12)\pi t}{T}\Big),
\]
with convergence in $L^2(\Omega)$ for each fixed $t$, and also in
$L^2(\Omega\times[0,T])$. Notice that for our purpose, the state space is every single piece of $\mathbb{R}$ in the orthogonal decomposition of the ambient space $\mathbb{R}^d$ that contains the data manifold, w.r.t. the standard basis.
\end{theorem}
\begin{proof}
    For the standard Brownian motion, we have $\mathbb{E}[W_sW_t]=\min(s,t)$ and $\mathbb{E}[W_t]=0$. Thus the covariance kernel is given by
    \begin{equation}
        K(s,t)=\min(s,t)\quad \text{for }s,t\in[0,T]
    \end{equation}
    Define the covariance operator
    \begin{equation}
        C:L^2([0,T])\to L^2([0,T]):f\mapsto \left[t\mapsto \int_0^TK(s,t)f(s)ds\right]
    \end{equation}
    Since $K$ is continuous and symmetric on $[0,T]^2$, the operator $C$ is symmetric, positive and compact. By the spectral theorem for compact self-adjoint operators, $C$ has an orthonormal basis of eigenfunctions with nonnegative eigenvalues. Now suppose that $e$ is an eigenfunction with eigenvalue $\lambda>0$, then\begin{align}
        \lambda e(t)&=\int_0^T\min (s,t)e(s)ds\\
        &=\int_0^tse(s)ds+t\int_t^Te(s)ds
    \end{align}
    Taking derivative w.r.t. $t$, we have \begin{equation}
        \lambda e'(t)=te(t)+\int_t^Te(s)ds-te(t)=\int_t^Te(s)ds
    \end{equation}
    and the second order derivative is $\lambda e''(t)=-e(t)$. Thus, $e''(t)+\frac{1}{\lambda} e(t)=0$. At $t=0$, we have $\lambda e(0)=0$, and at $T$, we have $\lambda e'(t)=\int_T^Te(s)ds=0$. Solving this second order ODE (detailed computation omitted), we obtain normalized eigenfunctions
    \begin{equation}
        e_n(t)=\sqrt{\frac{2}{T}}\sin\left(\frac{(n-1/2)\pi t}{T}\right)
    \end{equation}
    with corresponding eigenvalue $\lambda_n=T^2/[(n-\frac{1}{2})^2\pi^2]$. The key construction in this proof is to define
    \begin{equation}
        \xi_n\coloneqq\frac{1}{\sqrt{\lambda_n}}\int_0^TW_te_n(t)dt
    \end{equation}
    Since $W_t$ has mean zero and $\xi_n$ is a linear functional of the Gaussian process $W$, $\xi_n$ is Gaussian with mean zero. To show that $\xi_n$ are i.i.d. standard Gaussian, it remains to understand
    \begin{align}
        \mathbb{E}[\xi_n\xi_m] &=\frac{1}{\sqrt{\lambda_n\lambda_m}}\int_0^T\int_0^T \mathbb{E}[W_sW_t]e_n(s)e_m(t)dsdt\\
        &=\frac{1}{\sqrt{\lambda_n\lambda_m}}\int_0^T\int_0^T K(s,t)e_n(s)e_m(t)dsdt\\
        &= \frac{1}{\sqrt{\lambda_n\lambda_m}}\langle Ce_n,e_m\rangle =\frac{\lambda_n}{\sqrt{\lambda_n\lambda_m}}\langle e_n,e_m\rangle=\delta_{mn}
    \end{align}
    Finally, it suffices to show that the partial sum $S_N(t)=\sum_{n=1}^N \sqrt{\lambda_n}\xi_ne_n(t) $ converges to $W_t$. For fixed time $t$, we have
    \begin{align}
        \mathbb{E}[|W_t-S_N(t)|^2]=\mathbb{E}[W_t^2]-2\sum_{n=1}^{N}\sqrt{\lambda_n}e_n(t)\mathbb{E}[W_t\xi_n]+\sum_{n=1}^N\sum_{m=1}^N\sqrt{\lambda_n\lambda_m}e_n(t)e_m(t)\mathbb{E}[\xi_n\xi_m]\label{err}
    \end{align}
    where $\mathbb{E}[W_t\xi_n]=\frac{1}{\sqrt{\lambda_n}}\int_0^TK(s,t)e_n(s)ds=\sqrt{\lambda_n}e_n(t)$ and $\mathbb{E}[\xi_n\xi_m]=\delta_{mn}$. Therefore, (\ref{err}) can be simplified as\begin{equation}
        \mathbb{E}[|W_t-S_N(t)|^2]=K(t,t)-\sum_{n=1}^N\lambda_ne_n(t)^2\label{eq:err2}
    \end{equation}
    But we know that the RHS converges to zero for every $t$ due to the Mercer expansion of kernel. For convergence in $L^2(\Omega\times [0,T])$, we integrate the error term (\ref{eq:err2}) above over $t$ to get
    \begin{align}
        \int_0^T\mathbb{E}[|W_t-S_N(t)|^2]dt &= \int_0^TK(t,t)dt-\sum_{n=1}^N\int_0^T\lambda_n e_n(t)^2dt\\
        &=T^2/2-\sum_{n=1}^N\lambda_n
    \end{align}
    which converges to zero by the trace identity for compact positive operators.
\end{proof}

\subsection{It\^o-GF Estimator for Optimal Control}\label{estprf}
In this subsection, we prove consistency of the Itô-GF estimator under the diffusion-schedule condition (\ref{itogfcond}).
\newcommand{\EP}{\mathbb E_{\mathbb P}}
\newcommand{\EQ}{\mathbb E_{\mathbb Q}}
\newcommand{\given}{\mid X_t=x}
\begin{theorem}[Reward guidance from conditional endpoint samples]
Let $I_t=\alpha_t Z_0+\beta_t Z_1$, where
$Z_0\sim\mathcal N(0,I)$ and $Z_1\sim p_1$ are independent.
Write $p_t=\operatorname{Law}(I_t)$ and
$b_t(x)=\mathbb E[\dot I_t\mid I_t=x]$.
Consider the SDE with marginals $p_t$,
\begin{equation}
dX_t=
\left(b_t(X_t)+\frac{\sigma_t^2}{2}\nabla\log p_t(X_t)\right)dt
+\sigma_t\,dW_t,
\end{equation}
with the schedule
\begin{equation}
\frac{\sigma_t^2}{2}
=
\frac{\dot\beta_t}{\beta_t}\alpha_t^2
-\dot\alpha_t\alpha_t>0. \label{itogfcond}
\end{equation}
Proposition~11.1 of \citep{mfm} shows that its transition
kernel to the endpoint agrees with the interpolant's
conditional endpoint law:
\begin{equation}
\operatorname{Law}(X_1\mid X_t=x)
=
\operatorname{Law}(Z_1\mid I_t=x).
\end{equation}
This is the property we need: it allows samples obtained
by continuing the SDE from $x$ to be used in the
conditional expectations defining the tilted
interpolant's velocity.

For the target distribution $p_1^r\propto p_1e^r$, let
\begin{equation}
m_t^r(x)
=
\frac{\mathbb E[X_1e^{r(X_1)}\mid X_t=x]}
     {\mathbb E[e^{r(X_1)}\mid X_t=x]},
\qquad
c_t=\dot\beta_t-\frac{\dot\alpha_t\beta_t}{\alpha_t}.
\end{equation}
The velocity of the tilted interpolant is
\begin{equation}
b_t^r(x)
=
\frac{\dot\alpha_t}{\alpha_t}x+c_t m_t^r(x).
\end{equation}
Consequently, independent exact endpoint samples
$Y_j=\Phi_{t,1}(x,W^j)$ give the consistent estimator
\begin{equation}
\widehat{\Delta b}_t^{\,Z}(x)
=
c_t\left(
\frac{\sum_{j=1}^{Z}Y_j e^{r(Y_j)}}
     {\sum_{j=1}^{Z}e^{r(Y_j)}}
-\frac1Z\sum_{j=1}^{Z}Y_j
\right)
\xrightarrow[Z\to\infty]{\mathrm{a.s.}}
b_t^r(x)-b_t(x).
\end{equation}
\end{theorem}
\begin{proof}
We first identify the velocity associated with the tilted
endpoint distribution, and then explain how to estimate it
using SDE continuations.

Let $\widetilde Z_1$ have density $p_1^r\propto p_1e^r$,
independently of $Z_0$, and consider the tilted interpolant
\begin{equation}
\widetilde I_t=\alpha_t Z_0+\beta_t\widetilde Z_1.
\end{equation}
Changing the endpoint density from $p_1$ to $p_1^r$
reweights the joint distribution of $(I_t,Z_1)$ by
$e^{r(Z_1)}$, up to normalization. In particular, its
conditional endpoint mean becomes
\begin{equation}
\mathbb E[\widetilde Z_1\mid\widetilde I_t=x]
=
\frac{\mathbb E[Z_1e^{r(Z_1)}\mid I_t=x]}
     {\mathbb E[e^{r(Z_1)}\mid I_t=x]}.
\end{equation}
By Proposition~11.1 of \citep{mfm}, the chosen diffusion
schedule gives
\begin{equation}
\operatorname{Law}(Z_1\mid I_t=x)
=
\operatorname{Law}(X_1\mid X_t=x).
\end{equation}
We may therefore replace the interpolant's conditional
expectations by those of the SDE. The conditional mean
above is consequently exactly $m_t^r(x)$.

To obtain the velocity, condition the identity
$\widetilde I_t=\alpha_t Z_0+\beta_t\widetilde Z_1$
on $\widetilde I_t=x$. This gives
\begin{equation}
\mathbb E[Z_0\mid\widetilde I_t=x]
=
\frac{x-\beta_t m_t^r(x)}{\alpha_t}.
\end{equation}
Substituting into the interpolant velocity formula yields
\begin{align*}
b_t^r(x)
&=
\mathbb E[
\dot\alpha_t Z_0+\dot\beta_t\widetilde Z_1
\mid\widetilde I_t=x]\\
&=
\frac{\dot\alpha_t}{\alpha_t}x
+
\left(
\dot\beta_t-\frac{\dot\alpha_t\beta_t}{\alpha_t}
\right)m_t^r(x).
\end{align*}
The same calculation for the original interpolant gives
\begin{equation}
b_t(x)
=
\frac{\dot\alpha_t}{\alpha_t}x+c_t m_t(x),
\qquad
m_t(x)=\mathbb E[X_1\mid X_t=x].
\end{equation}
Subtracting the two velocities therefore gives
\begin{equation}
b_t^r(x)-b_t(x)
=
c_t\bigl(m_t^r(x)-m_t(x)\bigr).
\end{equation}
Thus the required correction depends only on the
reward-weighted and unweighted conditional endpoint means.

Finally, because $\Phi_{t,1}$ is an exact It\^o map,
independent Brownian continuations produce independent
samples
\begin{equation}
Y_j=\Phi_{t,1}(x,W^j)
\sim\operatorname{Law}(X_1\mid X_t=x).
\end{equation}
The strong law gives
\begin{equation}
\frac1Z\sum_{j=1}^{Z}Y_j\longrightarrow m_t(x),
\end{equation}
and
\begin{equation}
\frac1Z\sum_{j=1}^{Z}Y_j e^{r(Y_j)}
\longrightarrow
\mathbb E[X_1e^{r(X_1)}\mid X_t=x],
\qquad
\frac1Z\sum_{j=1}^{Z}e^{r(Y_j)}
\longrightarrow
\mathbb E[e^{r(X_1)}\mid X_t=x]
\end{equation}
almost surely. Since the limiting denominator is positive,
the ratio of the latter two sample means converges to
$m_t^r(x)$. Substituting these limits into
$\widehat{\Delta b}_t^{\,Z}(x)$ proves the claim.

This is the MFM-GF construction in Eq.~(20) of
\citep{mfm}, with the baseline conditional mean also
estimated from the endpoint samples.
\end{proof}

For ODE sampling, we add $b_t^r-b_t$ to the original
velocity. For SDE sampling with the schedule above,
we add $2(b_t^r-b_t)$ to the original SDE drift;
see Eqs.~(17)--(19) of \citep{mfm}.

\textit{Implementation remark.}
For $\alpha_t=1-t$ and $\beta_t=t$, the condition requires
$\sigma_t^2=2(1-t)/t$, and the estimator coefficient reduces to
$2/[\sigma_t^2(1-t)]$.
Our implementation has $\sigma_t^2=2(1-t)$, therefore we use It\^o-GF as approximate guidance only instead of claiming that this is a strongly consistent MC estimator of $\nabla V_t(x)$.

\subsection{Proof of Theorem \ref{bel}}\label{belprf}
We now prove Theorem \ref{bel}:
\begin{proof}
    From \citep[Theorem D.1]{BEL} we know that
    \begin{align}
        \nabla_x \mathbb{E}_{\mathbb{P}}[F(X_T)\mid X_s=x]
        &= \mathbb{E}_{\mathbb{P}}\left[F(X_T) \int_s^T\alpha_{t\mid s}J^{\top}_{t\mid s}\sigma_t^{-1}dW_t \mid X_s=x\right].
    \end{align}
    Applying this identity, we get
    \begin{align}
        \nabla_x \log \mathbb{E}_{\mathbb{P}}[F(X_T)\mid X_s=x]
        &= \frac{\mathbb{E}_{\mathbb{P}}\left[F(X_T) \int_s^T\alpha_{t\mid s}J^{\top}_{t\mid s}\sigma_t^{-1}dW_t \mid X_s=x\right]}
        {\mathbb{E}_{\mathbb{P}}[F(X_T)\mid X_s=x]}.
    \end{align}
    Since $F$ is positive and we are dividing by the normalization constant $\mathbb{E}_{\mathbb{P}}[F(X_T)\mid X_s=x]$ on the right hand side, this is by definition the integral with respect to $\mathrm{d}\mathbb{Q} \propto F(X_T) \mathrm{d} \mathbb{P}$:
    \begin{align}
    \frac{\mathbb{E}_{\mathbb{P}}\left[F(X_T) \int_s^T\alpha_{t\mid s}J^{\top}_{t\mid s}\sigma_t^{-1}dW_t \mid X_s=x\right]}
    {\mathbb{E}_{\mathbb{P}}[F(X_T)\mid X_s=x]} 
    = 
    \mathbb{E}_{\mathbb{Q}}\left[\int_s^T\alpha_{t\mid s}J^{\top}_{t\mid s}\sigma_t^{-1}dW_t \mid X_s=x\right],
    \end{align}
    which concludes the proof.
\end{proof}

In the paper, we chose $\alpha_{t\mid s}\propto \sigma_t$ as our default BEL estimator. To normalize it correctly with $\int_s^1\alpha_{t\mid s}dt=1$, we define $\alpha $ as:
\begin{equation}
    \alpha_{t\mid s} \coloneqq \frac{3}{2}(1-t)^{1/2}/(1-s)^{3/2}. \label{defbel}
\end{equation} 

\section{Implementation and Experiment Details}\label{appendb}
\paragraph{Limitations.} At ImageNet scale, our strongest reward-steering results use an It\^o map trained with a pretrained Decoupled MeanFlow diagonal teacher. Improving the self-distilled ImageNet model further, especially relative to strong deterministic flow maps, remains an important direction.
\paragraph{Brownian simulation.} In all experiments, we first simulate standard Brownian increments on a uniform time grid over $[s,1]$. We use these increments both to compute the martingale path $M_t=\int_s^t\sigma_udW_u$, approximated by cumulative sums of $\sigma_t\Delta W_t$, and to extract the leading KL coefficients used as Brownian conditioning inputs. For low-dimensional GMM experiments we use MLPs, while for MNIST and ImageNet-256, we use a Scalable Interpolant Transformer (SiT) \citep{sit} whose backbone is DiT \citep{dit}.
\paragraph{1D-GMM.}
To verify that the neural network can interpret KL coefficients as Brownian information and that our It\^o map training algorithms work, we first train a one-dimensional model, where $X_0\sim \mathcal{N}(0,1)$ and $X_1\sim \frac{1}{2}\mathcal{N}(-1,1)+\frac{1}{2}\mathcal{N}(1,1)$. Although the drift term here is explicitly computable, we treat $X_1$ as the data distribution and use simulation-free LSD and PSD objectives to train the model, rather than the simulation-full MSE loss. In terms of architecture, we use a single 6-layer MLP with hidden size 256. We choose $\sigma_t$ to be $\sqrt{2(1-t)}$. The Brownian path is simulated on a time grid with 200 grid points, and then converted to 5 KL coefficients. To ensure that the all-time stochastic dynamics is fully captured by the model, instead of just logging the endpoint distribution, we log the one-shot generation from time zero to every grid point $t=i/200$. The trajectory formed this way $(\hat{X}(0,\frac{i}{200},X_0,W))_{i=0,...,199}$ aligns well with the true SDE roll-out.
\paragraph{2D-GMM training.} For the 2-dimensional Gaussian mixture model, we train a single 12-layer MLP with hidden size 512. We treat the target distribution $p_1=\frac{1}{3}\mathcal{N}((-3,-3),0.5I)+\frac{1}{3}\mathcal{N}((0,0),0.5I)+\frac{1}{3}\mathcal{N}((3,3),0.5I)$ as a data distribution, set $\sigma_t=\sqrt{2(1-t)}$ and train an It\^o map $\hat{G}_{s,t}(x,W)$ with Lagrangian self-distillation. The Brownian motion is simulated on a 2-dimensional space, and the time interval $[0,1]$ is cut into 200 discretizations. We then convert the Brownian motion into 5 KL coefficients per dimension. The goal of this 2D-GMM experiment is to evaluate whether the proposed control estimators steer samples toward a known posterior distribution. We define the tilted distribution to be the posterior of a linear observation \begin{equation}
    p(x\mid 1.2x_1-0.8x_2+0.2\varepsilon=-1)
\end{equation}
for a standard Gaussian $\varepsilon$. This setup matches the corresponding Meta Flow Maps experiment, enabling a direct comparison. Starting from 4096 noise samples, we use different control estimators to guide the stochastic dynamics. We compare the sample distribution under each steering estimator against the true posterior, and compute the sliced Wasserstein-2 distance (S-W2) and maximum mean discrepancy (MMD) between them.
\paragraph{MNIST training.} For MNIST, we train an It\^o map directly in pixel space using a SiT-B model with depth 12, 12 attention heads, hidden dimension 768, and patch size 4. For Brownian input into the transformer, we use 5 leading KL coefficients per pixel as Brownian conditioning. For inference-time steering, we use Jeril Sebastian's pretrained MNIST classifier model on Hugging Face \citep{sebastian_mnist_classifier} to define the class-mixture reward described in Section \ref{MNinfsteer}. The It\^o map model used in the steering experiments was trained with KL features. We compare It\^o-G, It\^o-GF, BEL, BEL-I, DPS, and an SMC baseline by measuring the L2 distance between the generated class histogram and the target class mixture.

\paragraph{SiT architecture for ImageNet-256 training.} For ImageNet-256, we use a SiT-XL backbone with depth 28, 16 attention heads, hidden dimension 1152, and patch size 2. We consider two variants: a self-distilled It\^o map trained in VA-VAE latent space, used for the posterior-sampling visualization in Figure \ref{fig:posterior_sampling_row}, and a teacher-guided It\^o map trained in SD-VAE latent space using a Decoupled MeanFlow diagonal teacher, used for reward steering. In both variants, we use 10 leading KL coefficients per latent coordinate as Brownian conditioning. We use a batch size of 512, and a learning rate schedule: lr begins with $7.5e-4$ and gets a multiplier of 0.995 per 500 training steps until it reaches a minimum value of $2e-4$. The diagonal times are sampled uniformly from $[0,1]$; while the off-diagonal times are sampled using a logit-normal distribution: we sample $Z_t\sim\mathcal{N}(0,I)$ and $Z_s\sim \mathcal{N}(0,I)$ independently. Then define $t=\text{sigmoid}(0.6+Z_t)$ and $s=t \cdot\text{sigmoid}(Z_s)$. We also note that the uniform-half-plane sampling of $(s,t)$ works equally well. For classifier free guidance, when distilling the DMF teacher, we directly call the DMF checkpoint with no demixing (i.e. cfg scale is fixed at 2.5) and we note that training at a single cfg scale saves compute while keeping optimal performance. Another benefit of such distillation is that we could set class drop-out to be zero. When training from scratch, we sample the cfg scale $w$ from $\text{Unif}[1,5]$ and apply it to the off-diagonal training and evaluation. The class drop-out probability is set to be 0.1.
\begin{figure}[H]
    \centering
    \includegraphics[width=\linewidth]{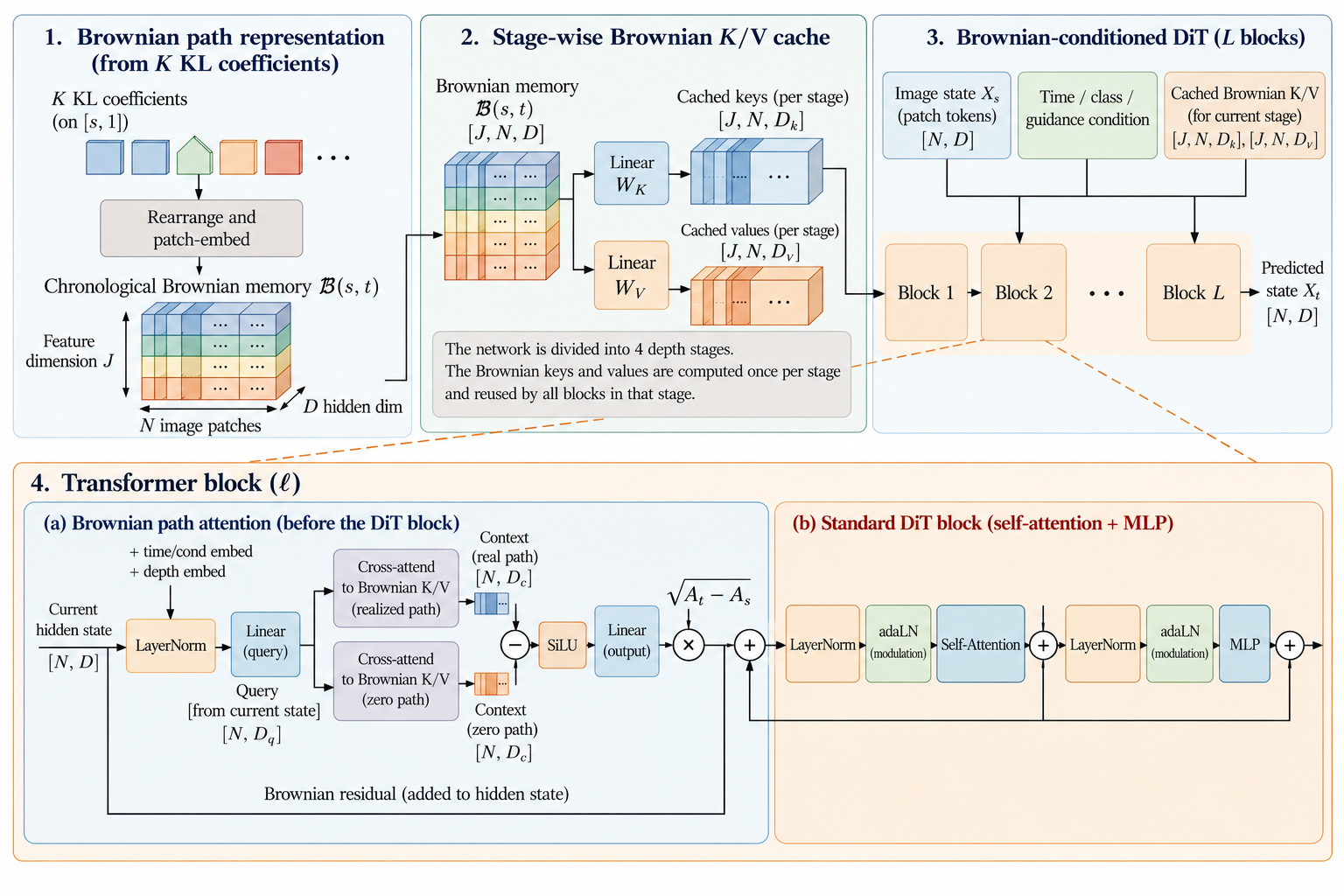}\vspace{-1em}
    \caption{Brownian-conditioned DiT architecture
    }
    \label{fig:psd18_architecture}
\end{figure}

\paragraph{Input format into the architecture (with KL representation).} For 1- and 2-dimensional Gaussian mixture models and MNIST model, the input into the MLP/SiT-B (for MNIST) consists of
\begin{equation}
    (s,t,X_s,5\text{ leading KL coefficients)}.
\end{equation}
For ImageNet-256, we use 10 KL coefficients and augment the SiT-XL backbone with a Brownian path-attention branch before every transformer block. \(K=10\) is chosen so that further increasing \(K\) produces negligible changes in the reconstructed smoothed Brownian trajectories. The $K$ KL coefficients describing the future path on $[s,1]$ are rearranged and patch-embedded into a chronological Brownian memory. We add interval and learned temporal positional embeddings so that the memory retains the ordering and location of the Brownian information along the queried interval. The 28 transformer blocks are divided into four depth stages. Within each stage, the Brownian memory is projected once into keys and values, which are cached and reused by all seven blocks in that stage. Every transformer block forms a fresh query from its current hidden state and attends to the cached Brownian memory, allowing the path to interact repeatedly with the evolving image representation.

For each block, the current hidden state is LayerNormed and linearly projected to a query. The same query attends to both the realized Brownian memory and a matched zero-path memory. We take the difference of the two nonlinear attention responses, project it back to the DiT hidden dimension, and scale it by $\sqrt{A_t-A_s}$, where
$A_t := \int_0^t \sigma_u^2\,du.$
This construction makes the Brownian branch vanish exactly for the zero path and on the diagonal $s=t$. We do not apply LayerNorm to the Brownian memory before the key/value projections. After the Brownian residual is added, the standard SiT block---self-attention followed by an MLP with adaLN conditioning---is applied. The martingale increment $M_t-M_s$ is never provided directly to the network and is used only in reconstructing the stochastic state through (\ref{eq:ansatz}).\vspace{-1em}

 \paragraph{Training with teacher-guidance (e.g. Decoupled MeanFlow \citep{dmf}).} 
 Many deterministic flow-map models provide a well-trained diagonal velocity $b_t(x)=\mathbb{E}[X_1-X_0\mid I_t=x]$. We use such pretrained models to accelerate It\^o map training by converting their diagonal velocity into the diagonal drift required by our SDE with one caveat: the diagonal in deterministic flow maps $b_t(x)$ is different from our drift: \begin{equation}
     G_{t,t}(x,W) = \mathbb{E}[X_1-2X_0\mid I_t=x].\label{ourdrift}
 \end{equation}
To address this issue, we make a simple observation here: \begin{equation}
     x = \mathbb{E}[tX_1+(1-t)X_0\mid I_t=x] \label{trivial}
 \end{equation}
which follows trivially from the definition. Solving simultaneous equations (\ref{ourdrift}) and (\ref{trivial}), we derive
\begin{equation}
    \mathbb{E}[X_0\mid I_t=x]=x-t\cdot b_t(x)
\end{equation}
and thus\begin{equation}
    G_{t,t}(x,W)=(1+t)b_t(x)-x
\end{equation}
The teacher-guided Lagrangian consistency training algorithm is given in Algorithm \ref{alg:teacher_lagrangian}.

  \begin{algorithm}
    \caption{Training with teacher guidance from a pretrained model using the Lagrangian loss.}
    \label{alg:teacher_lagrangian}

    \begin{algorithmic}[1]
      \Input $p_0 = p_{\text{noise}},\; p_1 = p_{\text{data}}$; model $\hat{G}_{s,t}(x,W)$; pretrained deterministic model $v^{\text{pretrained}}_{s,t}(x)$
      \Repeat
        \State \textbf{Compute diagonal loss:}
        \State Sample $t \sim U[0,1]$; simulate $X_0 \sim p_0$ and $X_1 \sim p_{\text{data}}$
        \State $I_t \gets tX_1 + (1-t)X_0$
        \State $\mathcal{L}_{\mathrm{SI}} \gets \|\hat{G}_{t,t}(I_t,\mathrm{zeros}) - ((1+t)v^{\text{pretrained}}_{t,t}(I_t) - I_t)\|_2^2$

        \State \textbf{Compute consistency loss:}
        \State Sample $(s,t)\sim U[0,1]^2$ and reorder so that $s<t$; simulate $X_0\sim p_0$ and $X_1\sim p_{\text{data}}$
        \State Simulate Brownian trajectory $(M_u)_{u\in[0,1]}$ and convert to Brownian features $W$
        \State $I_s \gets sX_1 + (1-s)X_0$
        \State $X_t \gets I_s + (t-s)\hat{G}_{s,t}(I_s,W) + M_t - M_s$
        \State $\mathcal{L}_{\mathrm{Lagrangian}} \gets
        \bigl\|\hat{G}_{s,t}(I_s,W)
        + (t-s)\partial_t \hat{G}_{s,t}(I_s,W)
        - ((1+t)v^{\text{pretrained}}_{t,t}(X_t) - X_t)\bigr\|_2^2$

        \State $\mathrm{loss} \gets \mathcal{L}_{\mathrm{SI}} + \lambda \mathcal{L}_{\mathrm{Lagrangian}}$
        \State Update model by taking one optimizer step
      \Until{convergence}
      \Output Trained It\^o map $\hat{X}_{s,t}(x,W) = x + (t-s)\hat{G}_{s,t}(x,W) + M_t - M_s$
    \end{algorithmic}
  \end{algorithm}
  
For the Decoupled MeanFlow model, note that their denoising process runs from $t=1$ to $t=0$, i.e. $X_1\sim p_{\text{noise}}$ and $x_0\sim p_{\text{data}}$. This means when querying their pretrained model checkpoint, the correct expression for $b_t$ is
\begin{equation}
    v_t^{\text{pretrained}}(x,y) = -\text{dmf\_}v(1-t,1-t,x,y)
\end{equation}
where dmf$\_v$ is the deterministic average velocity in reverse time and $y$ is the class label associated with $x$. One may also use pretrained teacher guidance for the progressive objective. The diagonal loss would be the same as in the algorithm above: \begin{equation}
    \mathcal{L}_{\text{SI}}\leftarrow\|\hat{G}_{t,t}(I_t,\text{zeros})-((1+t)v_{t,t}^{\text{pretrained}}(I_t)-I_t)\|_2^2
\end{equation} and the off-diagonal objective is the same as that for PSD.
\paragraph{DPS implementation.} For the MC-based DPS baseline, we estimate the terminal mean using independent Brownian continuations:
\[
\bar{x}_1^{(Z)}(x)
\coloneqq \frac{1}{Z}\sum_{j=1}^{Z}\widehat X(t,1,x,W^j),
\qquad
\widehat v_t^{\mathrm{DPS}}(x)
\coloneqq\nabla_{x_t}\left[\lambda r\left(\bar{x}_1^{(Z)}(x_t)\right)\right].
\]
The reward is evaluated on the averaged terminal prediction, so the guidance depends on the number of MC samples $Z$. For ImageNet, we use standard DPS which instead uses a denoised estimate obtained directly from the base model, without MC terminal sampling. Under our generator-NFE convention, this implementation requires 50 NFEs for a 50-step rollout.
\paragraph{Normalization of the steered drift estimator (NSDE).}
The estimators in Section 3.2 produce estimates of the optimal control direction at the current state. In the controlled SDE, the corresponding control drift is added to the base drift during sampling. In large-scale latent-space steering, the magnitude of the Monte Carlo control estimate is often not well calibrated: depending on the reward, timestep, and estimator, the control can be too small to noticeably affect the trajectory or too large to produce stable samples.

Following the NSDE used in Meta Flow Maps (F.3.4 in \citep{mfm}), we use a simple magnitude normalization for ImageNet steering. At each sampling step, we compute the base drift from the diagonal It\^o map and compute an estimated control drift using It\^o-G, It\^o-GF, BEL, or BEL-I. We then keep the direction of the control drift fixed, but rescale its norm to be comparable to the norm of the base drift. Concretely, if we denote the base SDE drift by 
\begin{equation}
    B_t(x) \coloneqq b_t(x)+\frac{\sigma_t^2}{2}\nabla \log p_t(x)
\end{equation}
and the estimated control drift by $\mathcal{V}_t(x)=\sigma_t^2\nabla V_t(x)$, then the drift used in the steered roll-out would be:\begin{equation}
    B_t(x) + \frac{\|B_t(x)\|}{\|\mathcal{V}_t(x)\|+\epsilon}\mathcal{V}_t(x)
\end{equation} where $\epsilon$ is a small constant.

This normalization is only a numerical stabilization of the controlled roll-out. It does not change the definitions of It\^o-G, It\^o-GF, BEL, or BEL-I, and it is not used during training. Unless otherwise stated, the ImageNet reward-steering results use NSDE.
\vspace{-0.5em}
\subsection{Compute Resources}\vspace{-0.5em}
We use NVIDIA H100 GPUs with 80 GB memory. The 1D and 2D GMM models each took less than 1 hour to train, the MNIST model took around 2 hours, and each ImageNet-256 It\^o map took about 192 GPU hours to train by using the Decoupled MeanFlow model checkpoint. For ImageNet steering, It\^o-G, It\^o-GF, and BEL-I took less than 1 minute per generation run, while BEL took about 5 minutes due to additional neural-network calls.\label{compute}\vspace{-0.5em}
\subsection{Broader Impact and Responsible Use}\label{broader}\vspace{-0.5em}
This paper studies stochastic generative modeling and inference-time control. The proposed It\^o map framework may improve the efficiency and flexibility of posterior sampling, stochastic simulation, and reward-guided image generation. Potential risks are similar to those of generative image models more broadly: stronger steering methods could be misused to generate misleading, harmful, or biased synthetic images, especially if combined with open-ended text-to-image systems. Our experiments are conducted on public benchmark datasets and pretrained models, and we do not release a deployment-ready image-generation system, new dataset, or new pretrained generative checkpoint. Responsible deployment should include appropriate content filtering, provenance or watermarking mechanisms, auditing of reward models and prompts, and human review in high-stakes applications. Our experiments do not involve human subjects, crowdsourcing, or personally identifiable information.\vspace{-0.5em}
\subsection{More ImageNet-256 Results}

\paragraph{Details for Table \ref{tab:reward} / ImageNet steering protocol.} \label{table1append}

We provide additional details for the ImageNet reward-steering results reported in Table \ref{tab:reward}. All inference-time steering experiments are performed in the SD-VAE latent space using an It\^o map trained with a Decoupled MeanFlow (DMF) diagonal teacher. In inference-time steering, the generative SDE drift $G_{t,t}(x)=\mathbb{E}[X_1-2X_0\mid I_t=x]$ as in (\ref{eq:gensde}) is obtained from the diagonal of the pretrained DMF model by \begin{equation}
    (1+t)v_t^{\text{pretrained}}(x)-x
\end{equation} and the stochastic control terms are then estimated using our trained It\^o map and its directional derivatives.

For each steering run, we use 50 discretization steps. At each time step $t_i$, the optimal control estimator is computed using 8 Monte Carlo Brownian trajectories on the remaining time interval $[t_i,1]$. We use reward scale $\lambda=3.0$ and apply NSDE normalization when applicable. 

For each class-prompt pair, we repeat the experiment 4 times and report the average reward. All methods are evaluated on the same set of classes, prompts, random initializations, and Brownian trajectories whenever applicable. Tables \ref{tab:imagereward_scores} and \ref{tab:hpsv2_scores} give the per-prompt ImageReward \citep{imagereward} and HPSv2 \citep{hpsv2} scores corresponding to the aggregate results in Table \ref{tab:reward}.

Tables \ref{tab:imagereward_scores} and \ref{tab:hpsv2_scores} show that It\^o-G gives the strongest average performance on both ImageReward and HPSv2. MFM-G is the closest baseline on ImageReward, while It\^o-G has a larger advantage on HPSv2. It\^o-GF improves over the unsteered SDE and DPS baselines and is competitive with Best-of-100 and MFM-GF. The BEL-family estimators give smaller gains, with more consistent improvements on ImageReward than on HPSv2. Overall, the per-prompt results support the main conclusion that It\^o map steering is strongest when reward gradients are available, while gradient-free estimators remain useful but more variable across settings.
\begin{table}[H]
\centering
\setlength{\tabcolsep}{1.5pt}
\begin{tabular}{@{}r L{0.26\linewidth} r r r r r r r r r r@{}}
\toprule
Cls. & Prompt & SDE & It\^o-GF & It\^o-G & BEL & BEL-I & DPS & Bo10 & Bo100 & MFM-GF & MFM-G \\
\midrule
663 & a stone monastery & -0.02 & 0.79 & 1.91 & 0.37 & 0.13 & -0.22 & 0.69 & 1.09 & 1.23 & 1.84 \\
929 & a pink popsicle & -0.40 & 1.77 & 1.98 & 1.29 & -0.15 & 0.43 & 1.09 & 1.82 & 1.49 & 1.96 \\
927 & a strawberry trifle & 1.23 & 1.81 & 1.98 & 1.41 & 1.07 & 1.21 & 1.78 & 1.86 & 1.83 & 1.96 \\
963 & a Pepperoni pizza & 0.05 & 0.60 & 1.14 & 0.30 & 0.12 & 0.26 & 0.42 & 0.55 & 0.54 & 1.05 \\
483 & a stone castle & 0.50 & 0.91 & 1.73 & 0.69 & 0.57 & 0.59 & 0.81 & 1.04 & 1.10 & 1.74 \\
933 & a very thick burger & 0.45 & 1.29 & 1.90 & 0.86 & 0.71 & 0.73 & 0.92 & 1.23 & 1.34 & 1.86 \\
360 & a very fat otter & -0.19 & 1.23 & 1.99 & 0.45 & 0.12 & 0.25 & 0.79 & 1.33 & 1.37 & 1.96 \\
974 & a geyser erupting straight upward with heavy steam & 0.69 & 1.31 & 1.85 & 1.04 & 0.71 & 0.88 & 1.05 & 1.35 & 1.41 & 1.79 \\
579 & an open grand piano & 0.70 & 1.15 & 1.93 & 0.95 & 0.45 & 0.55 & 1.29 & 1.58 & 1.63 & 1.90 \\
497 & a gothic church & 0.34 & 1.16 & 1.82 & 0.65 & 0.45 & 0.30 & 0.88 & 1.02 & 1.20 & 1.76 \\
849 & a ceramic teapot & 0.79 & 1.64 & 1.91 & 1.13 & 0.81 & 0.92 & 1.40 & 1.53 & 1.56 & 1.85 \\
934 & a mustard hotdog & 0.03 & 1.08 & 1.93 & 0.35 & -0.10 & -0.02 & 0.68 & 0.88 & 0.87 & 1.69 \\
29 & a floating axolotl & 0.36 & 1.20 & 2.00 & 0.50 & 0.54 & 0.38 & 1.14 & 1.60 & 1.73 & 2.00 \\
117 & a spiral nautilus & 1.23 & 1.81 & 1.99 & 1.77 & 1.47 & 1.51 & 1.84 & 1.89 & 1.88 & 1.98 \\
850 & a stitched teddy bear & 0.43 & 1.31 & 1.97 & 0.72 & 0.57 & 0.78 & 1.04 & 1.33 & 1.29 & 1.89 \\
247 & a droopy saint bernard & 0.49 & 1.71 & 1.98 & 1.39 & 0.34 & 0.24 & 1.54 & 1.67 & 1.67 & 1.97 \\
\midrule
\multicolumn{2}{@{}l}{Average} & 0.42 & 1.30 & \textbf{1.88} & 0.87 & 0.49 & 0.55 & 1.09 & 1.36 & 1.38 & 1.83 \\\bottomrule
\end{tabular}
\caption{ImageReward scores by prompt and method.}\label{tab:imagereward_scores}
\end{table}
\begin{table}[H]
\centering
\setlength{\tabcolsep}{1.5pt}
\begin{tabular}{@{}r L{0.26\linewidth} r r r r r r r r r r@{}}
\toprule
Cls. & Prompt & SDE & It\^o-GF & It\^o-G & BEL & BEL-I & DPS & Bo10 & Bo100 & MFM-GF & MFM-G \\
\midrule
663 & a stone monastery & 0.25 & 0.28 & 0.38 & 0.25 & 0.25 & 0.27 & 0.27 & 0.28 & 0.25 & 0.25 \\
929 & a pink popsicle & 0.23 & 0.28 & 0.35 & 0.23 & 0.25 & 0.24 & 0.26 & 0.29 & 0.24 & 0.23 \\
927 & a strawberry trifle & 0.27 & 0.31 & 0.37 & 0.26 & 0.24 & 0.25 & 0.29 & 0.31 & 0.27 & 0.27 \\
963 & a Pepperoni pizza & 0.28 & 0.30 & 0.37 & 0.28 & 0.28 & 0.28 & 0.30 & 0.31 & 0.27 & 0.26 \\
483 & a stone castle & 0.28 & 0.32 & 0.40 & 0.27 & 0.28 & 0.28 & 0.29 & 0.31 & 0.28 & 0.29 \\
933 & a very thick burger & 0.25 & 0.29 & 0.38 & 0.25 & 0.26 & 0.25 & 0.29 & 0.29 & 0.27 & 0.27 \\
360 & a very fat otter & 0.27 & 0.28 & 0.37 & 0.26 & 0.27 & 0.27 & 0.30 & 0.31 & 0.27 & 0.27 \\
974 & a geyser erupting straight upward with heavy steam & 0.25 & 0.29 & 0.40 & 0.25 & 0.25 & 0.26 & 0.27 & 0.29 & 0.27 & 0.27 \\
579 & an open grand piano & 0.24 & 0.26 & 0.37 & 0.23 & 0.23 & 0.23 & 0.27 & 0.28 & 0.25 & 0.24 \\
497 & a gothic church & 0.28 & 0.30 & 0.40 & 0.28 & 0.27 & 0.28 & 0.30 & 0.31 & 0.27 & 0.27 \\
849 & a ceramic teapot & 0.24 & 0.26 & 0.37 & 0.23 & 0.23 & 0.23 & 0.28 & 0.28 & 0.27 & 0.26 \\
934 & a mustard hotdog & 0.25 & 0.30 & 0.37 & 0.26 & 0.27 & 0.26 & 0.29 & 0.30 & 0.26 & 0.26 \\
29 & a floating axolotl & 0.24 & 0.27 & 0.39 & 0.24 & 0.24 & 0.24 & 0.27 & 0.28 & 0.25 & 0.24 \\
117 & a spiral nautilus & 0.24 & 0.27 & 0.37 & 0.24 & 0.25 & 0.25 & 0.27 & 0.29 & 0.26 & 0.25 \\
850 & a stitched teddy bear & 0.24 & 0.32 & 0.41 & 0.25 & 0.26 & 0.26 & 0.30 & 0.32 & 0.28 & 0.26 \\
247 & a droopy saint bernard & 0.26 & 0.31 & 0.39 & 0.26 & 0.24 & 0.25 & 0.29 & 0.30 & 0.27 & 0.26 \\
\midrule
\multicolumn{2}{@{}l}{Average} & 0.25 & 0.29 & \textbf{0.38} & 0.25 & 0.25 & 0.26 & 0.28 & 0.30 & 0.26 & 0.26 \\
\bottomrule
\end{tabular}
\caption{HPSv2 scores by prompt and method.}\label{tab:hpsv2_scores}
\end{table}

\paragraph{Steering for class 980 volcano, prompt \textit{a volcano erupting with magma}.} Average SDE unsteered reward: 0.107.
\begin{figure}[H]
    \centering
    \includegraphics[width=0.8\linewidth]{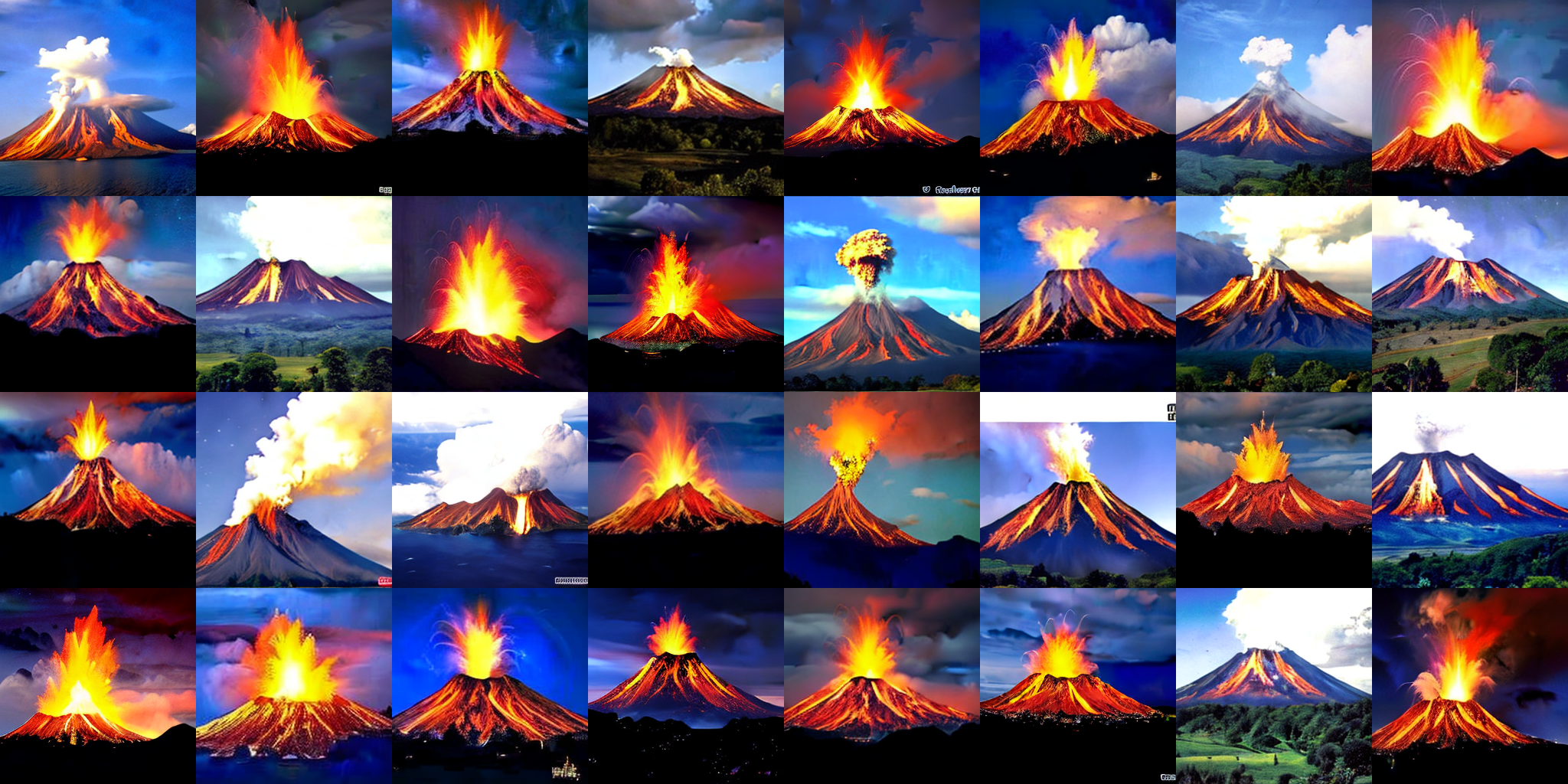}
    \caption{It\^o-G steered outputs using the same initialization $X_0$. 8 MC samples, reward scale $\lambda=3$, with ImageReward. Average steered reward 1.899.}
\end{figure}
\newpage
\begin{figure}[H]
    \centering
    \includegraphics[width=0.8\linewidth]{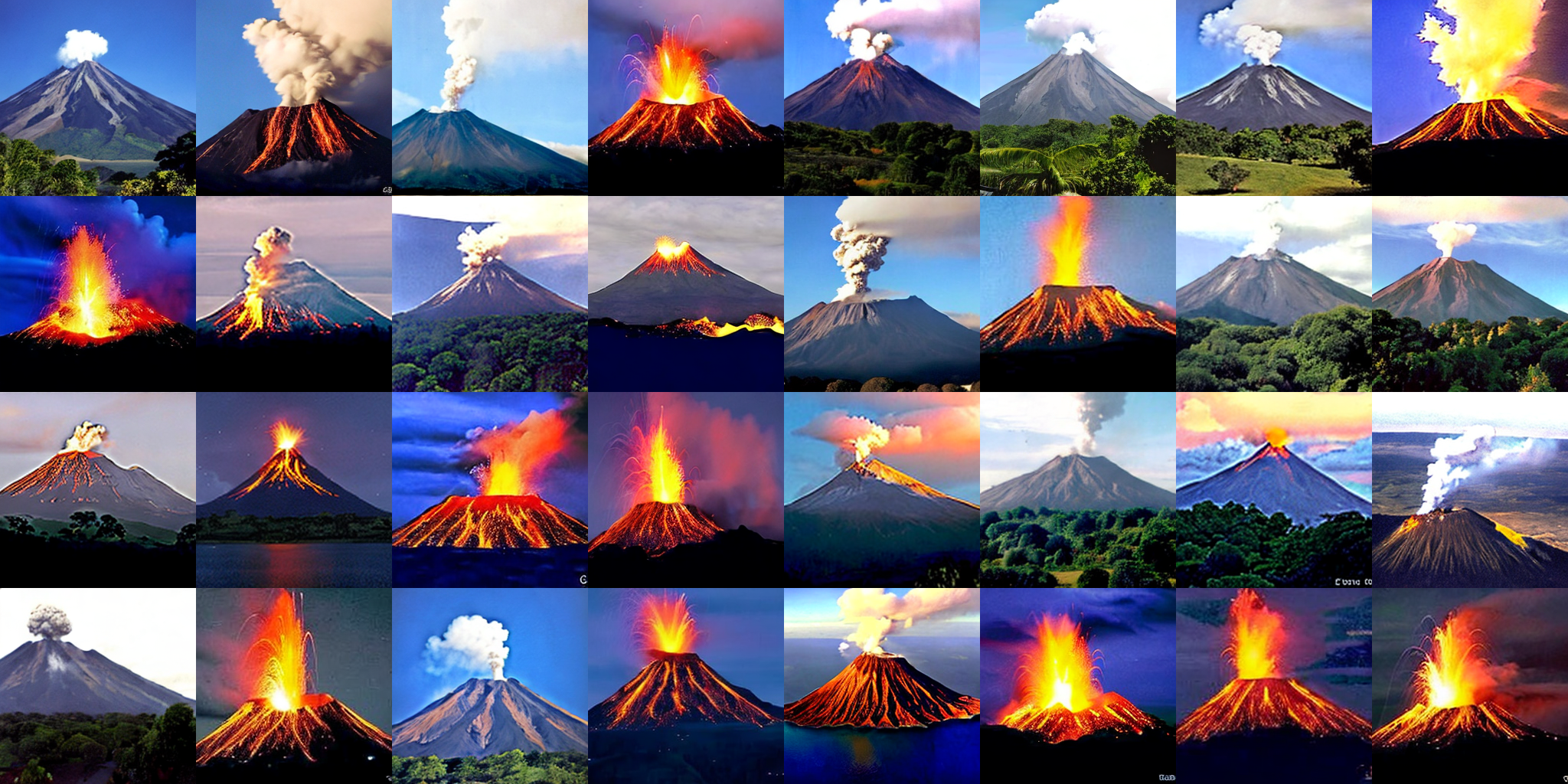}
    \caption{It\^o-GF steered outputs using the same initialization $X_0$. 8 MC samples, reward scale $\lambda=3$, with ImageReward. Average steered reward 1.394.}
\end{figure}

\begin{figure}[H]
    \centering
    \includegraphics[width=0.8\linewidth]{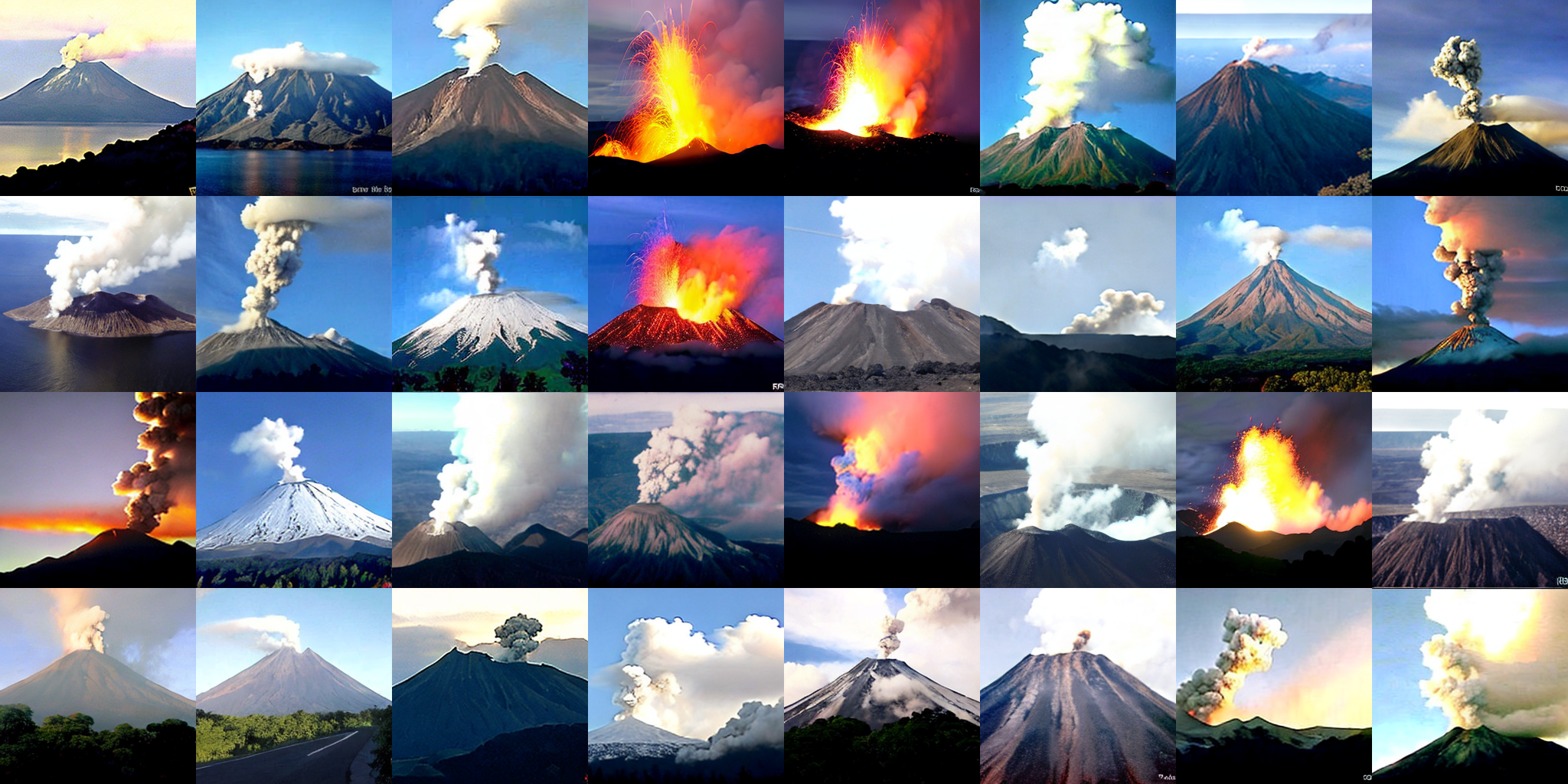}
    \caption{BEL steered outputs using the same initialization $X_0$. 8 MC samples, reward scale $\lambda=3$, with ImageReward. Average steered reward 0.723.}
\end{figure}
\begin{figure}[H]
    \centering
    \includegraphics[width=0.8\linewidth]{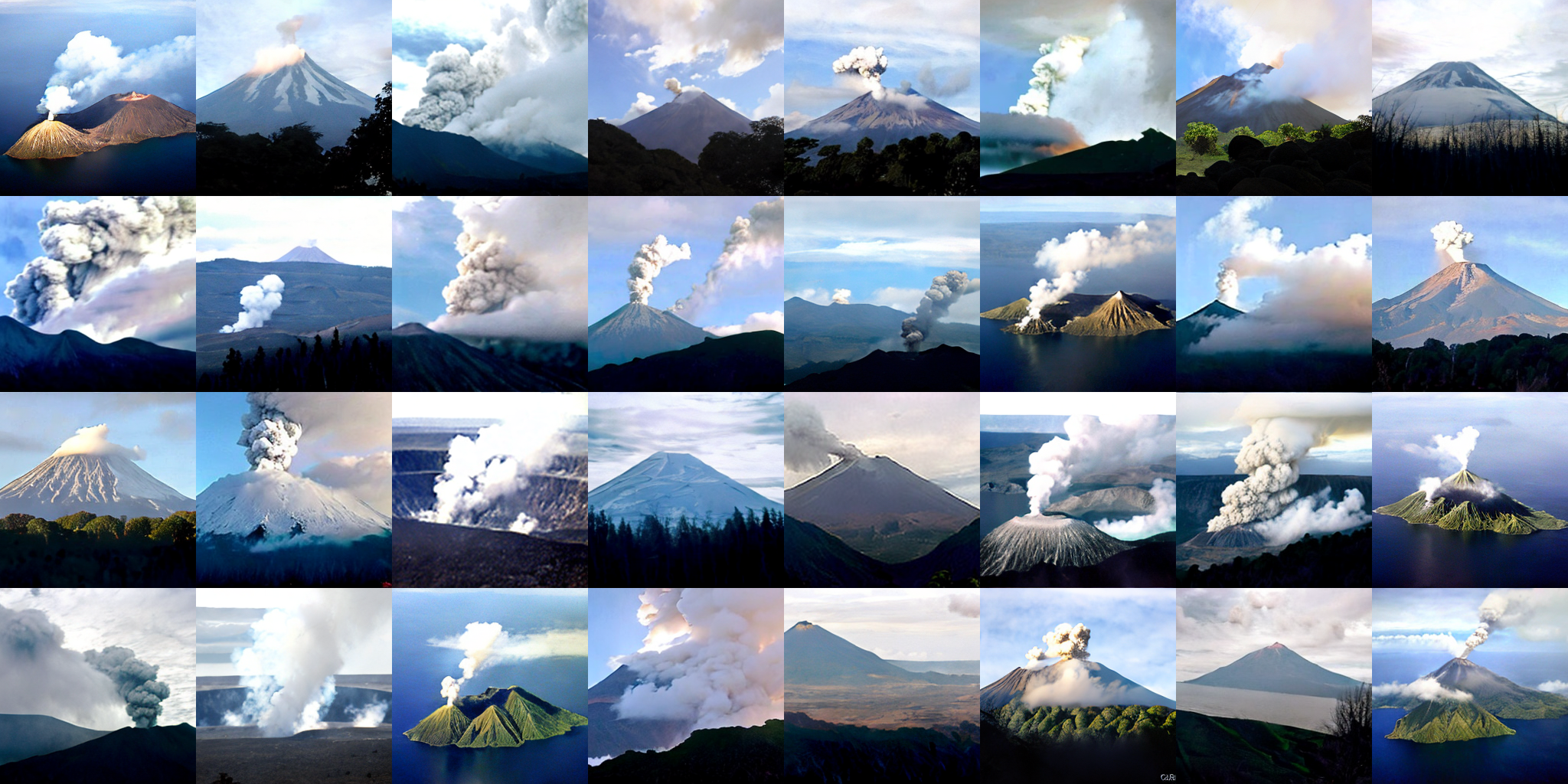}
    \caption{BEL-I steered outputs using the same initialization $X_0$. 8 MC samples, reward scale $\lambda=3$, with ImageReward. Average steered reward 0.342.}
\end{figure}




\end{document}